\documentclass{article} %

\usepackage[main,final,nonatbib]{neurips_2025}
\usepackage[numbers,sort&compress]{natbib}

\usepackage{hyperref}

\usepackage[boxruled,lined]{algorithm2e}

\usepackage{tikz}

\usepackage{url}

\usepackage{graphicx} %
\usepackage{subfigure}

\usepackage[normalem]{ulem}

\usepackage{amssymb}
\usepackage{amsmath}
\usepackage{amsthm}
\usepackage{thmtools,thm-restate}

\usepackage{booktabs}

\usepackage{xcolor}
 \definecolor{dgreen}{rgb}{0.00,0.49,0.00}
 \definecolor{dorange}{rgb}{1,0.4,0.00}
 \definecolor{dblue}{rgb}{0,0.08,0.75}
 \definecolor{BurntOrange}{HTML}{F7921D}
\hypersetup{
    colorlinks = true,
    citecolor=dgreen,
    urlcolor=dblue,
    linkcolor=dblue
}

\usepackage{stackengine}
\usepackage{scalerel}
\usepackage{xcolor}

\usepackage{version}
\makeatletter
\newcommand{\SetVersionColor}[2]{%
  \colorlet{ver@clr@#1}{#2}%
  \expandafter\def\csname vercolor@#1\endcsname{ver@clr@#1}%
}
\newcommand{\VersionColorName}[1]{%
  \@ifundefined{vercolor@#1}{black}{\csname vercolor@#1\endcsname}%
}
\newcommand{\EnableVersion}[1]{%
  \includeversion{#1}%
  \expandafter\def\csname ver@#1\endcsname{1}%
  \expandafter\AtBeginEnvironment\expandafter{#1}{\begingroup\color{\VersionColorName{#1}}}%
  \expandafter\AtEndEnvironment\expandafter{#1}{\endgroup}%
}
\newcommand{\DisableVersion}[1]{\excludeversion{#1}\expandafter\let\csname ver@#1\endcsname\relax}
\newcommand{\IfVersionTF}[3]{\expandafter\ifx\csname ver@#1\endcsname\relax #3\else #2\fi}
\newcommand{\V}[2]{\IfVersionTF{#1}{\textcolor{\VersionColorName{#1}}{#2}}{}}
\makeatother

\usepackage{macro}

\usepackage[nameinlink,capitalize]{cleveref}

\title{
  Dynamic Regret Reduces to Kernelized Static Regret
}

\author{
  Andrew Jacobsen\thanks{Equal contribution.}\\
  Universit\`{a} degli Studi di Milano\\
  Politecnico di Milano \\
  \texttt{contact@andrew-jacobsen.com}
  \And
  Alessandro Rudi\footnotemark[1]\\
  Bocconi University\\
  \texttt{alessandro.rudi@sdabocconi.it}
  \And
  Francesco Orabona\\
  King Abdullah University of Science and Technology (KAUST)\\
  Thuwal, 23955-6900, Kingdom of Saudi Arabia\\
  \texttt{francesco@orabona.com}
  \And
  Nicol\`{o} Cesa-Bianchi\\
  Universit\`{a} degli Studi di Milano\\
  Politecnico di Milano \\
  \texttt{nicolo.cesa-bianchi@unimi.it}
}

\begin{document}
\date{}

\maketitle

\begin{abstract}
    We study dynamic regret in online convex optimization, where the objective is to achieve low cumulative loss relative to an arbitrary benchmark sequence. By observing that competing with an arbitrary sequence of comparators $u_{1},\ldots,u_{T}$ in $\mathcal{W}\subseteq\mathbb{R}^{d}$ can be reframed as competing with a  \emph{fixed} comparator \emph{function} $u:[1,T]\to \mathcal{W}$,  we cast dynamic regret minimization as a \emph{static regret} problem in a \emph{function space}. By carefully constructing a suitable function space in the form of a Reproducing Kernel Hilbert Space (RKHS), our reduction enables us to recover the optimal $R_{T}(u_{1},\ldots,u_{T}) = \mathcal{O}(\sqrt{\sum_{t}\|u_{t}-u_{t-1}\|T})$ dynamic regret guarantee in the setting of linear losses, and yields new scale-free and directionally-adaptive dynamic regret guarantees. Moreover, unlike prior dynamic-to-static reductions---which are valid only for linear losses---our reduction holds for \emph{any} sequence of losses, allowing us to recover $\mathcal{O}\big(\|u\|^2_{\mathcal{H}}+d_{\mathrm{eff}}(\lambda)\ln T\big)$ bounds when the losses have meaningful curvature, where $d_{\mathrm{eff}}(\lambda)$ is a measure of complexity of the RKHS. Despite working in an infinite-dimensional space, the resulting reduction leads to algorithms that are computable in practice, due to the reproducing property of RKHSs.
  \end{abstract}

\section{Introduction}

This paper introduces new techniques for \emph{Online Convex Optimization} (OCO),
a framework for designing and analyzing algorithms which
learn on-the-fly from a stream of data \citep{Gordon99b,zinkevich2003online,cesa2006prediction,orabona2019modern, Cesa-BianchiO21}.
Formally, consider $T$ rounds of interaction between a learner and the
environment. In each round, the learner chooses $\wt\in \ww$ from
a convex set $\ww\subseteq\R^{d}$, the environment reveals
a convex loss function $\ell_{t}:\ww\to\R$, and the learner
incurs a loss of $\ell_{t}(\wt)$.
The classic objective in this setting is to minimize the
learner's \emph{regret} relative to any fixed benchmark $\cmp\in \ww$:
\begin{equation*}
  \Regret(\cmp) := \textstyle\sumtT (\ell_{t}(\wt)-\ell_{t}(\cmp))~.
\end{equation*}
In this paper, we study the more general problem of minimizing the learner's
regret relative to any \emph{sequence}
of benchmarks $\cmp_{1},\dots,\cmp_{T}\in \ww$~\citep{HerbsterW98b,HerbsterW01,zinkevich2003online}:
\begin{equation*}
  \DRegret(\cmp_{1},\dots,\cmp_{T})
  :=\textstyle\sumtT (\ell_{t}(\wt)-\ell_{t}(\cmp_{t}))~.
\end{equation*}
This objective is typically referred to as \emph{dynamic}
regret, to distinguish it from the special case where
the comparator sequence is fixed $\cmp_{1}=\dots=\cmp_{T}$ (referred to as
\emph{static} regret).
Intuitively, dynamic regret captures a notion of \emph{non-stationarity} in learning problems. Problem instances where $\cmp_1=\cdots=\cmp_T$ model classic problem settings, wherein there is a fixed ``solution'' whose performance we want to emulate, while a time-varying comparator sequence models problem settings where the learner needs to continuously adapt to a changing environment
in which the solution is time-varying.
The complexity of a given comparator sequence is typically characterized by its
\emph{path-length}:
\begin{align*}
    P_T = \textstyle\sum_{t=2}^T\norm{\cmp_t-\cmp_\tmm}~.
\end{align*}
Clearly, if the path-length is large there is no hope to obtain low dynamic regret. The goal is thus to 
obtain performance guarantees that gracefully \emph{adapt} to the level of non-stationarity. 
For instance,
in the setting of $G$-Lipschitz losses and a bounded domain $D=\sup_{x,y\in \ww}\norm{x-y}$,
the minimax optimal dynamic regret guarantee 
is of the order of $\cO(G\sqrt{(D^{2}+DP_T) T})$,
which scales naturally with the complexity of the benchmark sequence and recovers the
optimal $\cO(GD\sqrt{T})$ static regret guarantee
when the comparator is fixed.
  In unbounded domains (e.g., $\ww=\R^{d}$) these bounds would be vacuous, so
  the guarantee should instead be adaptive to $M:=\max_{t}\norm{\cmp_{t}}$.
  In this case an analogous guarantee of
  $\tilde\cO(G\sqrt{(M^{2}+MP_{T})T})$
  can be achieved at the expense of additional logarithmic terms.
  Throughout the paper we focus on the unbounded setting.


\begin{figure}
\centering
\begin{tikzpicture}[scale=.5]
  \begin{scope}[shift={(0,0)}]
    \coordinate (A1) at (1,0);
    \coordinate (B1) at (2,2);
    \coordinate (C1) at (3,1);
    \coordinate (D1) at (4,3);
    \coordinate (E1) at (5,0);
    \coordinate (F1) at (6,2);
    \draw[->] (-0.5,0) -- (6.5,0) node[right] {$t$};
    \draw[->] (0,-0.5) -- (0,3.5) node[above] {$u_t$};
    \foreach \x in {1,...,6} { \draw (\x,0.1) -- (\x,-0.1) node[below] {\x}; }
    \foreach \p in {A1,B1,C1,D1,E1,F1} { \fill[black] (\p) circle (2.5pt); }
  \end{scope}

  \draw[->, very thick] (7.5,1.5) -- (9,1.5) node[midway, above] {Transform};

  \begin{scope}[shift={(10,0)}]
    \coordinate (A2) at (1,0);
    \coordinate (B2) at (2,2);
    \coordinate (C2) at (3,1);
    \coordinate (D2) at (4,3);
    \coordinate (E2) at (5,0);
    \coordinate (F2) at (6,2);
    \draw[->] (-0.5,0) -- (6.5,0) node[right] {$t$};
    \draw[->] (0,-0.5) -- (0,3.5) node[above] {$u(t)$};
    \foreach \x in {1,...,6} { \draw (\x,0.1) -- (\x,-0.1) node[below] {\x}; }
    \foreach \p in {A2,B2,C2,D2,E2,F2} { \fill[black] (\p) circle (2.5pt); }
    \draw[orange, thick,  domain=1:6, samples=400]
      plot (\x,{4/15*(\x-1)^5 - 13/4*(\x-1)^4 + 83/6*(\x-1)^3 - 97/4*(\x-1)^2 + 77/5*(\x-1)
        - 0.005*((\x-1)*(\x-2)*(\x-3)*(\x-4)*(\x-5)*(\x-6)*(\x-3.5)^4)});
    \draw[blue, thick] plot coordinates {(A2) (B2) (C2) (D2) (E2) (F2)};
    \draw[black, thick, dotted, domain=1:6, samples=200]
      plot (\x,{4/15*(\x-1)^5 - 13/4*(\x-1)^4 + 83/6*(\x-1)^3 - 97/4*(\x-1)^2 + 77/5*(\x-1)});
  \end{scope}
\end{tikzpicture}
\caption{
\begin{small}Transformation from a sequence of comparators to a function. Many
  functions may implement the transformation. 
    In \Cref{sec:translation-invariant}
  we will see that under mild assumptions on the chosen function space $\hh$ we can always find a function $\cmp\in\hh$ such that $\cmp(t)=u_t$ for all $t$ and
  $\norm{\cmp}_\hnorm^{2}= O\big(\sum_{t=2}^T\norm{\cmp_t-\cmp_\tmm}\big)$.\end{small}
}
\label{fig:transform}
\end{figure}

\paragraph{Contributions. } In this work we introduce a new framework for reducing
dynamic regret minimization to static regret minimization.
Our key insight is that competing with a \emph{sequence}
$\cmp_{1},\ldots,\cmp_{T}$ in $\ww$ can be equivalently
framed as competing with some fixed \emph{function} $\cmp(\cdot)$
such that $\cmp(t)=\cmp_{t}$ for all $t$. In this view,
we effectively transform dynamic regret minimization
over a domain $\ww\subseteq\R^{d}$ into a \emph{static} regret minimization
problem over a domain of \emph{functions}, depicted graphically in \Cref{fig:transform}.

The choice of the function space is crucial, as it controls the trade-offs of the resulting algorithm. To complete the construction, we carefully design a rich family of function spaces which embed the comparator sequence in a way that (1) optimizes the inherent trade-offs of the function class to achieve optimal dynamic regret guarantees and (2) ensures that the resulting algorithm is computable in practice, despite being stated as an infinite-dimensional optimization problem.
Indeed, the family we design is an instance of a Reproducing Kernel Hilbert Space (RKHS), a well-studied class of 
functions endowed with the familiar structure of a Hilbert space.
The reduction to learning in an RKHS is particularly natural
in the context of online learning---the vast majority of modern online learning
theory is developed for \emph{static regret minimization in Hilbert spaces}, so
our reduction enables the use of the familiar
online learning toolkit while also allowing us to draw upon
deep connections
between dynamic regret minimization, kernel methods, and signal processing theory.

In the linear losses setting,  our construction
enables us to achieve the optimal dynamic regret guarantees of $\cO(\sqrt{MP_T T})$ up to poly-logarithmic terms. Notably, the resulting algorithm is naturally \emph{horizon independent}, and is easily extended to a \emph{scale-free} version.
These are the first algorithms that obtain the optimal $\sqrt{P_T}$ dependence without prior knowledge of the horizon $T$ natively, without resorting to the doubling trick. 
Our reduction also enables us to derive new 
\emph{directionally-adaptive} guarantees, 
which
scale as
$\tilde \cO\Big(\sqrt{\deff(\lambda)\big(\norm{\cmp}^2_\hnorm + \sumtT \inner{\gt,\cmp_t}^2\big)}\Big)$, where $\norm{\cmp}^2_\hnorm$ and  $\deff(\lambda)$ are measures of the complexity of the comparator function and complexity of the function class $\hh$ respectively. 

Interestingly, because our reduction only involves viewing the comparator sequence through a different lens, it holds for \emph{any} sequence of loss functions, 
 contrasting prior works which 
 are valid only for linear losses \citep{zhang2023unconstrained,jacobsen2024equivalence}.
We show that this allows us to account for loss 
\emph{curvature} and obtain 
$\cO(\lambda\norm{\cmp}_\hh^2 + \deff(\lambda)\log T)$ dynamic regret
in the context of  strongly-convex, exp-concave, and improper linear regression settings.

\paragraph{Related Works. } 
Our work is most directly related to a recent thread of research in the linear loss setting initiated by \citet{zhang2023unconstrained}. Their strategy approaches
dynamic regret from a signal processing perspective, wherein the
comparator sequence is stacked into a high-dimensional ``signal'' $\tilde\cmp=\Vec{\cmp_1,\ldots,\cmp_T}\in\R^{dT}$, and $1$-dimensional static regret algorithms are employed to learn the coefficients of a basis of features which decompose that signal, leading to 
$\cO(\sqrt{MP_T T})$ dynamic regret via a carefully chosen
dictionary of features.
\citet{jacobsen2024equivalence} generalize this perspective by 
designing static regret algorithms that are applied directly in this high-dimensional space, and derive the $\mathcal{O}(\sqrt{MP_T T})$ bound
by choosing a suitable dual-norm pair in this space, such that 
$\norm{\tilde\cmp} = \mathcal{O}(\sqrt{P_T})$. 
Our work further extends this perspective by interpreting the comparator sequence as samples of 
a \emph{function} in an RKHS $\hh$, and designing algorithms 
which obtain suitable static regret guarantees in function space. 
The reduction of \citet{jacobsen2024equivalence} can in fact be understood as a special case of our framework by choosing the discrete RKHS $\hh$ associated with the Dirac kernel.

More broadly, the concept of dynamic regret was originally introduced by \citet{HerbsterW98b,HerbsterW01}. Later, \citet{zinkevich2003online} showed that OGD 
naturally obtains $\cO(P_T\sqrt{T})$ dynamic regret and
\citet{yang2016tracking} 
showed that $\cO(\sqrt{DP_T T})$ can be achieved
when prior knowledge of $P_T$ is available. The first to achieve the $\cO(\sqrt{DP_T T})$ rate \emph{without} prior knowledge of $P_T$ was
\citet{zhang2018adaptive}, who also 
proved a matching lower bound, and the analogous bound of
$\cO(\sqrt{MP_{T}T})$ has been achieved up to logarithmic terms in unbounded settings
\citep{jacobsen2022parameter,luo2022corralling,jacobsen2023unconstrained,zhang2023unconstrained}.
There have also been several refinements to the result, replacing the $T$ factor
with data-dependent quantities such as $\sumtT \norm{\gt}^2$ or
$\sum_t \sup_x\abs{\ell_t(x)-\ell_\tmm(x)}$
\citep{cutkosky2020parameter,campolongo2021closer,hall2015online}.
Going beyond linear losses, various improvements in adaptivity can be obtained
when the losses are smooth or exp-concave, such as replacing the $T$ factor with 
$\sum_t\ell_t(\cmp_t)$ or $\sum_t \sup_w\norm{\grad\ell_t(\w) - \grad\ell_\tmm(\w)}^2$ 
\citep{zhao2020dynamic,zhao2022efficient,zhao2024adaptivity}.
In the squared loss setting $\ell_t(w)=\half (y_t-w)^2$, minimax optimal rates 
of $R_T(\cmpy_1,\ldots,\cmpy_T)=\cO(C_T^{2/3}T^{1/3})$ have been obtained where $C_T=\sum_{t=2}^T\abs{\cmpy_t-\cmpy_\tmm}$ is the path-length of the benchmark predictions \citep{KoolenMBAY15,sadhanala2016total,BabyW19,baby2021optimal,baby2023secondorder,zhang2025nonstationary}.

\section{Preliminaries}%
\label{sec:bg}

\paragraph{Notations. }
Hilbert spaces are denoted by upper case calligraphic letters.
Given a Hilbert space $\hh$, we denote the associated
inner product by $\inner{\cdot,\cdot}_{\hh}$.
We denote $\Lin(\hh,\ww)$ the space of linear operators from
$\hh$ to $\ww$. $\Lin(\hh,\ww)$ is itself a Hilbert space when equipped with the
Hilbert-Schmidt inner product, $\inner{A,B}_\HS=\Tr{A^{*}B}$,
where $A^{*}\in\Lin(\ww,\hh)$ is the adjoint of $A$.
The subdifferential set of a function $f$ at $x$
is denoted by $\partial f(x)$. We will occasionally abuse notation and 
write $\grad f(x)$ to mean an arbitrary element of $\partial f(x)$. We will
denote by $[T]$ the set $\{1, 2, \dots, T\}$. The Fourier transform of a
function $Q$ is denoted $F[Q](x)=\int_{\R} Q(\omega)e^{-2\pi i x\omega}d\omega$
and, when clear from context, we will generally abbreviate $F[Q](x)=:\hat Q(x)$.

\paragraph{Reproducing Kernel Hilbert Spaces. }
Let $\hh = \Set{h:\xx\to\R}$ be a Hilbert space of functions
a on compact set $\xx$.
The space $\hh$ is a RKHS~\citep{aronszajn1950theory} 
if there exists a positive definite function $k:\xx\times\xx\to\R$ such that $k(\cdot,x)\in\hh$ for all $x\in\xx$, that has the {\em reproducing property}, i.e., we have $f(x) = \inner{f, k(\cdot, x)}_{\hh}$ for all $f\in\hh$ and $x\in\xx$.
The function $k$ is called the \emph{kernel function} associated
with $\hh$ and the function $\phi(x)=k(\cdot,x)$ is called the
\emph{feature map}. It is known that the kernel function uniquely
characterizes the RKHS $\hh$.
A kernel is \emph{universal}
if it can approximate any real-valued continuous function on $\xx$ to
arbitrary accuracy. Many of the standard kernel functions are universal, including the
Gaussian RBF kernel, the Mat\'{e}rn  kernel, and inverse multiquadratic kernel.
\emph{All kernels considered in this work are universal kernels.}
For a detailed introduction to kernel methods, see, e.g.,  \citet{scholkopf2002learning,wendland2004scattered,berlinet2011reproducing,paulsen2016introduction}.

We will often be interested in functions taking
values in $\ww\subseteq\R^{d}$. In this case the usual RKHS machinery
extends in a straight-forward way via a coordinate-wise
extension. Indeed, we can represent $\w:\xx\to\R^{d}$
as a tuple $\w = (\w_{1},\ldots,\w_{d})$ such that
$\w_{i}\in\hh$ for each $i$.\footnote{
  This connection can be made more
  formally via Riesz representation theorem, see \Cref{app:bg} for details.}
This naturally leads to
an operator-based version of the reproducing property:
\begin{align*}
  \w(t)=(\w_{1}(t),\ldots,\w_{d}(t))=\big(\inner{\w_{1},\phi(t)}_{\hh},\ldots,\inner{\w_{d},\phi(t)}_{\hh}\big) = \W\phi(t)\in\ww,
\end{align*}
where $\W\in\Lin(\hh,\ww)$.
The space $\Lin(\hh,\ww)$ is itself a Hilbert space when
equipped with the Hilbert-Schmidt norm, and under the coordinate-wise extension above we have
$\norm{\W}_{\knorm}^{2}=\norm{\w}^{2}_{\hh^{d}}=\sum_{i=1}^{d}\norm{\w_{i}}_{\hh}^{2}$.
Moreover, observe that
when $d=1$ this setup simply reduces back to the usual setup, wherein
$\norm{\W}_{\HS}=\norm{w}_{\hh}$ and $\w(t)=\inner{\w,\phi(t)}_{\hh}$.

For notational clarity we will refer to functions $w(\cdot)\in\hh$
and their values
$\w(t)\in\ww$ with lower-case letters, and their
representation $\W\in\Lin(\hh,\ww)$ using
the upper-case. We will typically use the notation $\wt$ in place of $\w(t)$ when
referring to the evaluations of $\w(t)$ at discrete time-points $t\in[T]$.

\begin{figure}
\vspace{-2.5em}
\SetAlCapHSkip{0.5em}\setlength{\algomargin}{0.5em}
\begin{algorithm}[H]
  \SetAlgoLined
  \textbf{Input: }
  Domain $\ww\subseteq\R^{d}$,
  feature map $\phi:\xx\to\hh$, algorithm
  $\cA$ defined on $\WW$\\
  \For{$t=1:T$}{
    Receive $\Wt\in\WW$ from $\cA$ and play $\wt=\Wt\phi(t)\in\ww$\\
    Observe loss function $\ell_{t}:\ww\to\R$
    and incur loss $\ell_{t}(\wt)$\\

    Send $\elltilde_t \colon W\mapsto\ell_{t}(\W\phi(t))$ to $\cA$ as the $t^{\text{th}}$ auxiliary loss 
  }
  \caption{Kernelized Online Learning}
  \label{alg:kernel-oco}
\end{algorithm}
\end{figure}

\section{\SecRedux}%
\label{sec:redux}

In this section, we present the main tool that we will use to
develop dynamic regret guarantees for online learning.
The key idea is to interpret the comparator sequence
$\cmp_{1},\dots,\cmp_{T}\in\ww$ as the evaluations of
a function $\cmp(\cdot)$ at the discrete time-points
$t\in[T]$, allowing us to re-frame dynamic regret minimization
as a \emph{static regret minimization in function space}.

Note that most existing work in online learning revolves around learning in Hilbert spaces, not general function spaces, so if we hope to leverage these existing tools we should
embed the comparator sequence in a \emph{Hilbert space of functions}. In particular, 
our approach will be to embed the comparator sequence in a Hilbert space $\hh$ of functions representable by a 
reproducing kernel $k(s,t)$ and feature map $\phi:\xx\to\hh$.
Note that this is always possible by selecting a universal kernel on $\xx$. Our reduction is conceptually shown in \Cref{alg:kernel-oco} and the following
theorem shows that the dynamic regret \wrt{} comparator sequence $\cmp_1,\dots,\cmp_T$
in $\ww$ is equivalent to \emph{static} regret \wrt{} 
a function $\cmp(\cdot)\in\hh$ 
on the auxiliary loss sequence  
$\elltilde_t \colon W\mapsto\ell_{t}(\W\phi(t))$.

\begin{theorem}[Dynamic Regret via Kernelized Static Regret]\label{thm:RKHS-SR-solves-DR}
  Let $\xx$ be a compact set, $\ww\subseteq\R^{d}$,  let $\hh$ be an RKHS with associated
  feature map $\phi:\xx\to\hh$, and for any $\W\in\WW$ let $\elltilde_{t}(\W)=\ell_{t}(\W\phi(t))$.
  Let $\W_{1},\dots,\W_{T}$ be an arbitrary sequence in $ L(\hh,\ww)$ and suppose that
  on each round we play $\wt = \W_{t}\phi(t)\in\ww$. Then,
  for any comparator sequence $\cmp_{1},\dots,\cmp_{T}$ in $\ww$ and $\Cmp\in\WW$ 
  satisfying $\cmp_{t}=\Cmp\phi(t)$ for all $t$,
  \begin{align*}
  \DRegret(\cmp_{1},\ldots,\cmp_{T})
  &=
    \sumtT (\ell_{t}(\wt)-\ell_{t}(\cmp_{t}))
  =
    \sumtT (\tilde\ell_{t}(\Wt)-\tilde\ell_{t}(\Cmp))
  =:
  \RKHSRegret(\Cmp)\ .
  \end{align*}
\end{theorem}

Note that the reduction holds for \emph{any}
operator $\Cmp\in\WW$ which interpolates
the comparator sequence, $\cmp_{t}=\Cmp\phi(t)\ \forall t$. 
Hence, we can always let $\cmp(\cdot)\in\hh$ be the minimum norm function in 
$\hh$ which interpolates these points, and take $\Cmp$ to be its representation in $\WW$.
In fact, since $k$ is a
universal kernel, we can 
assume that $\cmp(\cdot)\in\hh$  approximates 
\emph{any} continuous function on $\xx\subseteq [1,T]$ to arbitrary accuracy, so the assumption that $\cmp(\cdot)$ 
lives in an RKHS $\hh$ does not actually restrict the functions that we can
compare against in a significant way.
We will see a concrete example of an RKHS $\hh$ which reconstructs arbitrary comparator sequences
in later sections (e.g., \Cref{thm:discrete-pl}).

Our reduction bears a strong resemblance to the reduction
recently proposed by \citet{jacobsen2024equivalence}, which
works by embedding the comparator sequence in $\R^{dT}$
by simply ``stacking'' the comparator sequence into
one long vector, $\tilde\cmp=\Vec{\cmp_{1},\ldots,\cmp_{T}}\in\R^{dT}$.
In fact, we show in \Cref{app:recovering} that our framework
\emph{precisely} recovers the reduction in \citet{jacobsen2024equivalence}
by choosing the discrete RKHS $\hh$ associated with the Dirac kernel.
However,
notice in particular that this means that their reduction
is inherently tied to \emph{finite}-dimensional features,
whereas ours enables \emph{infinite}-dimensional features.
As we will see in
\Cref{sec:olo}, this distinction is key to obtaining
the optimal path-length dependencies in a horizon-independent manner.
Moreover, note that the regret equality in \citet{jacobsen2024equivalence} holds
\emph{only} in the context of linear losses, $\ell_{t}(w)\mapsto\inner{g_{t},\w}$,
whereas in our framework the regret equality holds for any sequence of losses
$\ell_{1},\ldots,\ell_{T}$.
We will see that this
distinction is important in \Cref{sec:curvature}---for example, our reduction allows us to
preserve the curvature of the losses needed to obtain $\mathcal{O}(\norm{\cmp}^2_\hh+\deff(\lambda) \ln T)$
bounds when the original losses $\ell_{t}$ are strongly convex
or exp-concave, where $\deff(\lambda)$ is a measure of complexity of the RKHS.


\section{\SecOLO}%
\label{sec:olo}

We first consider 
the setting of online linear optimization. In this setting, on round $t$ the learner receives 
linear loss $\ell_t(w)=\inner{g_t,\w}_{\ww}$, so recalling the reduction in the previous section defines 
the auxiliary loss as 
$\elltilde_t\colon \W\in \Lin(\hh,\ww)\mapsto \ell_t(\W\phi(t))\in\R$,
we have
\begin{align*}
\elltilde_t(\W) 
= \inner{\gt, \W\phi(t)}_{\ww}=\inner{\gt\otimes\phi(t), \W}_\knorm=\inner{\Gt,\W}_\knorm,
\end{align*}
where $\Gt=\gt\otimes \phi(t)\in\Lin(\hh,\ww)$
is the rank one operator such that $(\gt\otimes\phi(t))(h) = \inner{\phi(t),h}_\hh\gt$ for any $h\in\hh$.
As such,  it is important that the base algorithm
facilitates an application 
of the kernel trick to avoid explicitly evaluating 
the feature map $\phi(t)$, which may be infinite-dimensional in general.
To help make things concrete and provide intuitions,
the following example shows that many
of the common
algorithms based on Follow the Regularized Leader (FTRL)
  with a radially-symmetric regularizer $w\mapsto\Psi_{t}(\norm{\w})$
  are amenable to the kernel trick.
  This class captures many of the fundamental regularizers
  in online learning, such as quadratic regularizers
  and the ``linearithmic'' \citep{orabona2021parameterfree} regularizers
  $\Psi_t(\norm{w})\approx \norm{w}\sqrt{t\log(\norm{\w}/\alpha+1)}$
  associated with the
  comparator-adaptive regret guarantees
  that the key result of this section (\Cref{prop:pf-static-olo}) will
  be derived from.

\begin{example}\label{ex:kernel-ftrl}
(Kernel Trick for Kernelized FTRL)
    Let $\g_1,\ldots,\g_T$ be a sequence in $\ww$
    and
    let $\Gt=\gt\otimes\phi(t)\in\WW$ for all $t$. 
    Let $\theta_{t}=-\sum_{s=1}^{\tmm}\Gs$,
    $V_{t}=\sum_{s=1}^{\tmm} \norm{\Gt}^{2}_{\knorm}$, let $\Psi_t(\cdot;V_t)$ be a
    convex function with differentiable Fenchel conjugate $\Psi_{t}^{*}$,
    and
    consider the following FTRL update:
    \begin{align*}
        \W_t 
        =
        \argmin_{\W\in\Lin(\hh,\ww)}\inner{\theta_t, \W}+\Psi_t(\norm{\W}_{\knorm}; V_t)
        = \grad_{\theta}\Psi_{t}^*\brac{\norm{\theta_{t}}_{\knorm}; V_{t}} = \frac{\theta_{t}}{\norm{\theta_{t}}_{\knorm}}(\Psi_{t}^{*})'(\norm{\theta_{t}}_{\knorm}; V_{t})~.
    \end{align*}
    Then, $V_t=\sum_{s=1}^\tmm\norm{\gs}^2_\ww k(t,t)$ (\Cref{lemma:hs-norm}),
    $\norm{\theta_{t}}_{\knorm}^2= \sum_{s,s'=1}^\tmm k(s,s')\inner{\gvar{s},\gvar{s'}}_\ww$ (\Cref{lemma:hs-sum-norm}), and
    on round $t$, \Cref{alg:kernel-oco} plays
    \begin{align*}
    \wt = \Wt\phi(t) = -\frac{(\Psi_{t}^{*})'(\norm{\theta_{t}}_{\knorm};V_{t})}{\norm{\theta_{t}}_{\knorm}}\sum_{s=1}^\tmm k(s,t)\gs~.
    \end{align*}
\end{example}
The example shows that many common instances of FTRL
can be kernelized without explicit computation of the feature map.
The example also demonstrates an important consideration when applying static
regret decompositions of this nature: the update described above would require
$\mathcal{O}(t)$ time and memory to implement in general,
while existing algorithms for dynamic regret can often be implemented using
$\mathcal{O}(\ln T)$ computation and memory~\citep{zhang2018adaptive,jacobsen2022parameter,zhang2023unconstrained,zhao2024adaptivity}.
Luckily, there is already a deep and well-developed literature on efficient approximations for kernel methods
that can be leveraged to translate the algorithms developed from the kernelized
OCO point-of-view into more practically implementable algorithms \citep[see, e.g.,][]{sun2015review,liu2022random}. Since these extensions are already well-understood and since implementating these details would not yield any new insights in the current paper, we
will not consider them further here, focusing instead on
the theoretical development.

Now that we have seen how to translate an algorithm's updates 
to the kernelized setting, we turn now to 
how to translate its static regret guarantees 
into dynamic regret guarantees. 
The following result shows that an algorithm's 
kernelized static regret 
guarantee translates in a straight-forward way to a
dynamic regret guarantee in the original problem.
The proof is immediate by applying \Cref{thm:RKHS-SR-solves-DR}
and computing
$\norm{\Gt}_{\knorm}=\norm{\gt\otimes \phi(t)}_{\knorm}=\norm{\gt}_{\wwdual}\sqrt{k(t,t)}$
by \Cref{lemma:hs-norm}.

\begin{restatable}{theorem}{OLOReduction}\label{thm:olo-reduction}
Let $\mathcal{A}$ be an online learning algorithm defined on Hilbert space
$\vv$.
Suppose that for any sequence of convex loss functions $h_1,\ldots,h_T$ on $\cV$, $\mathcal{A}$ obtains a bound on the static regret of the form
$
\RKHSRegret(\Cmp)
\leq B_T\big(\|\Cmp\|_\vv, \|\nabla h_1(\W_{1})\|_{\vvdual}, \dots, \|\nabla h_T(\W_{T})\|_{\vvdual}\big)
$
for any comparator $\Cmp \in \vv$ and some function
$B_T: \R_{\ge 0}^{T+1} \to \R$, where $\grad h_{t}(W_{t})\in\partial h_{t}(W_{t})$
for all $t$. 
If we apply $\cA$ in $\cV=\WW$ with
$\norm{\cdot}_\cV=\norm{\cdot}_{\knorm}$,
then for any
sequence $\cmp_{1},\ldots,\cmp_{T}$ in $\ww$
and
$\Cmp\in\WW$ satisfying
$\cmp_{t}=\Cmp\phi(t)$ for all $t$, \Cref{alg:kernel-oco} with ${\cal A}$ guarantees
\[
\DRegret(\cmp_{1},\ldots,\cmp_{T})
\leq B_T\left(\|\Cmp\|_{\knorm}, \|g_{1}\|_{\wwdual} \sqrt{k(1,1)}, \dots, \|g_{T}\|_{\wwdual} \sqrt{k(T,T)}\right),
\]
where $g_{t}\in \partial\ell_{t}(\wt)$ for all $t$,
and $k(\cdot, \cdot)$ is the {\em reproducing kernel} associated to the space $\hh$.
\end{restatable}
The value of the lemma is that it enables us
to immediately translate static regret guarantees from
OLO to guarantees in our RKHS formulation of dynamic regret,
wherein the complexity of the comparator sequence is
measured by the RKHS norm $\norm{\Cmp}_{\knorm}=\norm{\cmp}_{\hnorm}$. For instance, if
we simply apply the standard (sub)gradient descent guarantee
to \Cref{thm:olo-reduction}
we get
\begin{align*}
  \DRegret(\cmp_{1},\dots,\cmp_{T})
  &=
    \RKHSRegret(\Cmp)
    \le
    \frac{\norm{\cmp}_{\hnorm}^{2}}{2\eta}+\frac{\eta}{2}\sumtT \norm{\gt}^{2}_{\wwdual}k(t,t)~.
\end{align*}
Optimally tuning $\eta$
yields
$\DRegret(\cmp_{1},\dots,\cmp_{T})\le \norm{\cmp}_{\hnorm}\sqrt{\sumtT \norm{\gt}^{2}_{\wwdual}k(t,t)}$,
so achieving the optimal $\mathcal{O}\brac{\sqrt{P_{T}T}}$ has effectively
been reduced to the problem of \emph{designing a kernel} such that\footnote{Here and in the following the $\tilde{\mathcal{O}}$ notation will hide polylogarithmic factors.}
$\norm{\cmp}_{\hnorm}=\sqrt{\sum_{i=1}^d\norm{u_i}^2_{\hh}}=\tilde{\mathcal{O}}(\sqrt{P_{T}})$ while controlling $k(t,t)$.
We will see in \Cref{sec:optimal-pl} that this can be accomplished
by using a carefully chosen translation-invariant kernel.

In the above argument, the optimal choice of $\eta$
would require prior knowledge of $\norm{\cmp}_{\hnorm}$ and cannot be chosen in general. Luckily, 
there are static regret algorithms which can \emph{adapt} to the comparator norm automatically
to obtain the 
optimal trade-off up to logarithmic terms \citep{mcmahan2012noregret,mcmahan2014unconstrained,orabona2016coin,cutkosky2018black,jacobsen2022parameter}.
For our purposes
we will refer to an algorithm $\cA$ defined on Hilbert 
space $\vv$ as \emph{parameter-free} if 
for any sequence $G$-Lipschitz loss functions 
$h_1,\ldots,h_T$ and any $\Cmp\in\vv$, 
$\cA$ guarantees
\begin{align}
    \RKHSRegret(\Cmp)= \tilde{ \mathcal{O}}\brac{\norm{\Cmp}_{\vv}\sqrt{\textstyle\sumtT \norm{\grad h_t(\Wt)}^{2}_{\vvdual}}}~.\label{eq:pf}
\end{align}
There are many existing algorithms which satisfy this property; we provide a concrete example and its updates
in our framework for completeness in \Cref{app:pf-static-olo}.
Using such an algorithm in \Cref{alg:kernel-oco}
immediately yields the following regret guarantee.

\begin{restatable}{proposition}{PFStaticOLO}\label{prop:pf-static-olo}
  Let $\cA$ be a static regret algorithm for Hilbert spaces
  satisfying \Cref{eq:pf}.
  For any $G>0$, any sequence of $G$-Lipschitz losses
  $\ell_{1},\ldots,\ell_{T}$, and any sequence $\cmp_{1},\ldots,\cmp_{T}$ in $\ww$,
  and $\Cmp\in\WW$ satisfying $\cmp_{t}=\Cmp\phi(t)$,
  \Cref{alg:kernel-oco} applied with $\cA$ guarantees
  \begin{align}
    \DRegret(\cmp_{1},\ldots,\cmp_{T})=\RKHSRegret(\Cmp) = \tilde{\mathcal{O}}\brac{\norm{\Cmp}_{\knorm}\sqrt{\textstyle\sumtT \norm{\gt}^{2}_{\wwdual}k(t,t)}},\label{eq:olo-guarantee}
  \end{align}
  where $\gt\in\partial\ell_t(\wt)$ for all $t$.
\end{restatable}


\label{sec:translation-invariant}

\subsection{\SecOptimalPL}%
\label{sec:optimal-pl}
\Cref{prop:pf-static-olo} demonstrates a clear 
trade-off between the RKHS norm $\norm{\cmp}_\hh$ and the associated kernel 
$k(t,t)$ induced by the choice of function space:  smaller the RKHS norms correspond to larger function spaces, hence higher values of $k(t,t)$.
In order to obtain the optimal $\cO(\sqrt{P_T T})$ dynamic regret, we need to design a
kernel such that $\norm{\Cmp}_\knorm=\norm{\cmp}_\hnorm = \cO(\sqrt{P_T})$ and $k(t,t)$ is controlled for all $t$.
Throughout this section, we will assume for simplicity that $d=1$ but note that
the extension to $d>1$ is straightforward via the coordinate-wise extension in \Cref{sec:bg}.

Recall that a \emph{translation invariant} kernel over $\R$ is characterized
by the Fourier transform of its \emph{spectral density} $Q$
\citep{wendland2004scattered}, where $Q$ is a real non-negative integrable
function. In particular, a translation invariant kernel and its associated norm are
\begin{equation}\label{eq:def-translation-inv-kernel}
  k(t,t') = \hat Q(t-t')= \int_{\R}Q(\omega)e^{-2\pi i \omega(t-t')}d\omega,\quad
  \norm{f}_{\hh}^{2}=\int_{\R} \frac{\abs{\hat f(\omega)}^{2}}{Q(\omega)}d\omega,
\end{equation}
where we use the short-hand notation $\hat g= F[g]$ to denote the Fourier
transorm of a function $g(\cdot)$.
The intuition behind focusing on translation-invariant kernels
  is that the associated
  norm provides a natural connection to the $\sqrt{MP_{T}}$ dependencies we would like
  like to achieve. Indeed, observe that with spectral density
  $Q(\omega)\approx 1/\omega$,
  we would have via Parseval's identity and the fact that
  $F[f'(x)](\omega)=2\pi i \omega \hat f(\omega)$ that
  \begin{align}
    \norm{\cmp}^{2}_{\hh} \approx \int_{\R}\omega\hat \cmp(\omega)\overline{\hat \cmp(\omega)}d\omega \le \int_{\R} \abs{\grad \cmp(t)} \abs{\cmp(t)}dt\le \sup_{t}\abs{\cmp(t)}\norm{\grad \cmp}_{L^{1}},\label{eq:ideal-ti-kernel-norm}
  \end{align}
  which is the continuous-time analogue of $MP_{T}$.
  The key challenge is to choose an integrable $Q(\omega)$ which
  suitably trades off
  the comparator norm $\norm{\cmp}^{2}_{\hh}$ and the magnitude of the associated
  kernel entries $k(t,t)$. Unfortunately, these trade-offs are non-trivial using
  standard translation-invariant kernels, as shown in the following example.
\begin{example}[Existing kernels lead to sub-optimal trade-offs \citep{wendland2004scattered}]
At first glance, the spline kernel seems like a natural candidate since it has $\|u\|^2_{\cal H} = \|\Partial{u} \|^2_{L^2}=\cO(\sum_t\norm{u_t-u_\tmm}^2)$ (\Cref{thm:discrete-squared-pl}).
However, the spline kernel also has $k(t,t) = t$, leading to a suboptimal rate
in \Cref{prop:pf-static-olo}. On the other hand, for the classical translation invariant kernels such as the Gaussian or the Matern kernels, we have $k(t,t) = \cO(1)$ but $\|u\|^2_{\hh} = \|u\|^2_{L^2} + \sum_{n \geq 1} c_n \|\Partial[n]{u}\|^2$ for $c_n$ positive and summable. In this case, note that $k(t,t)$ has the good rate but $\|u\|^2_{\hh} \geq \|u\|^2_{L^2}$ and $\|u\|^2_{L^2} = c^2 T$ already for constant comparators on $[0,T]$, $u(t) = c 1_{[0,T]}(t)$, $c > 0$, precluding the optimal rate.
\end{example}

Given the above, we next turn our attention to designing a new kernel that will achieve the desired trade-offs.
Since we need to find a delicate balance in the trade-off of $\norm{\cmp}_\hh$  and $k(t,t)$ to achieve optimal rates, 
in the first part of the section we first derive a result that identifies general sufficient conditions to bound the RKHS norm of a translation invariant kernel in terms of the \emph{continuous path-length} $\norm{\Partial{\cmp} }_{L^{1}}=\int \abs{\grad\cmp(t)}dt$ (\Cref{thm:rkhs-to-pl}). Then, in \Cref{prop:horizon-free-density}, we design an explicit kernel satisfying such conditions, leading to a trade-off of $\cO(\sqrt{\norm{\grad\cmp}_{L^1}T})$. Finally, 
 in \Cref{thm:discrete-pl} we show that under mild conditions (which are satisfied by the kernel in the \Cref{prop:horizon-free-density}), it is 
always possible to find an 
$u(\cdot)\in \hh$ such that $\cmp(t)=\cmp_t$ for all $t$ and
$\norm{\grad \cmp}_{L^1}= \cO(\sum_t\norm{\cmp_t-\cmp_\tmm}_\ww)$, so achieving dynamic regret scaling with $\norm{\cmp}_\hh=\cO(\norm{\grad\cmp}_{L^1})$ recovers the
usual path-length. The proof of the following theorem can be found in \Cref{app:rkhs-to-pl}

\begin{restatable}{theorem}{RKHSToPL}\label{thm:rkhs-to-pl}
Let $Q:\R \to \R_+$ be an integrable strictly positive even function on $\R \setminus \{0\}$ and such that $R(x) := 2\pi/(x(1+(x/2\pi)^{2m})Q(x))$ is also integrable for some $m \in \N$, $m\geq 1$. Let $k$ be defined in terms of $Q$ as in \cref{eq:def-translation-inv-kernel}.  Then $k$ is a translation invariant universal kernel with $k(t,t) \leq \|Q\|_{L^1}$ for all $t \in \R$. The RKHS $\hh$ associated to $k$ contains the space of finitely supported functions with bounded derivatives up to order $2m$, and moreover, for any $T > 0$ and any $2m$-times differentiable function $f$ that is supported on $[0, T+1]$,
\[
\|f\|^2_{\hh} ~~\leq~~ c(T) \;\|\Partial{f}\|_{L^1} \; \|f - \Partial[2m]{f}\|_{L^\infty}, 
\]
where $c(T) := \|F[R]\|_{L^1([-T-1, T+1])}$. If $R$ is monotonically decreasing on $(0, \infty)$, then,
\[
c(T) ~\leq~ \inf_{\alpha > 0} ~ 2 \pi (T+1)^2 \int_0^\alpha R(x) x dx + \frac{2}{\pi} \int_\alpha^\infty \frac{R(x)}{x} dx
~, \qquad \forall T > 0
\]
\end{restatable}

With this in hand, the following 
proposition provides an example of spectral density $Q$ which will leads to 
the desired dependency $\norm{\cmp}_\hh^{2} = \mathcal{O}(\norm{\grad\cmp}_{L^{1}})$, up to poly-logarithmic terms.
Proof can be found in \Cref{app:horizon-free-density}. 

\begin{minipage}{\columnwidth}
\begin{restatable}{proposition}{HorizonFreeDensity}\label{prop:horizon-free-density}
    Let $Q:\R \to \R_+$ be defined as
    \[
    Q(\omega) = \frac{1/4 \, \log\log \pi}{|\omega|~\,(1+ |\omega|^2/4 \pi^2)^{\frac{1}{4}} \,\, \log(\pi + |\omega|^{-\frac{1}{2}}) \,\, \log^2\log(\pi+|\omega|^{-\frac{1}{2}})}~.
    \]
    Then we can apply \cref{thm:rkhs-to-pl} with $m=1$: the function $k$ defined in terms of $Q$ as in \cref{eq:def-translation-inv-kernel} is a translation invariant kernel with
    $k(t,t) \leq 8\pi^2, ~ \forall t\in\R$; the associated RKHS norm satisfies
    \[
    \|f\|^2_\hh ~~\leq~~ c^2 ~\|\Partial{f}\|_{L^1} \, \|f - \Partial[2]{f}\|_{L^\infty} ~ (\ln (1+T) \ln \ln (1+T))^2,
    \]
    for any $f$ that is $2$-times differentiable and supported in $[0,T+1]$, where $T > 2$ and $c \leq (2\pi e)^2$.
\end{restatable}
\end{minipage}

A notable property of the kernel characterized by \Cref{prop:horizon-free-density}
is that it is
\emph{horizon independent}, requiring no upper bound on $T$ to control $\norm{f}_\hh^2$ and $k(t,t)$.
This is a non-trivial property to guarantee using existing methods without resorting to the doubling trick, which is well-known to perform poorly in practice.
  The intuitions behind the choice of $Q(\omega)$ follow from the discussion
  above: we would like to set $Q(\omega)\approx1/\abs{\omega}$ so that
  $\norm{\cmp}_{\hh}$ relates to the path-length via
  \Cref{eq:ideal-ti-kernel-norm}, but this would not be a valid choice
  because $Q(\omega)=1/\abs{\omega}$ is not integrable.
  \Cref{prop:horizon-free-density} adds a small bit of additional regularization
  to ensure that $Q(\omega)$ is integrable while remaining close to $1/\abs{\omega}$.
  We provide additional intuition on the choice of regularization
  in \Cref{app:spectral-intuition}.

A subtlety that we have glossed over thus far is
that the \emph{continuous path-length}, $\norm{\Partial{\cmp}}_{L^{1}}=\int \norm{\Partial{\cmp(t)}}_{\ww}dt$, does not
necessarily compare favorably to the
classic \emph{discrete path-length} $P_{T}=\sum_{t}\norm{\cmp_{t}-\cmp_{\tmm}}_{\ww}$ since the function may vary wildly between the interpolated points.
The next theorem shows that we can always find a function such that$\norm{\Partial{\cmp}}_{L^{1}}=\mathcal{O}(P_{T})$. Proof can be found in \Cref{app:discrete-pl}. 
\begin{restatable}{theorem}{DiscretePL}\label{thm:discrete-pl} Let
  $v_1,\dots, v_T \in \R^d$ and let $\hh$ be the RKHS associated to kernel $k$
  contain finitely supported functions with bounded derivatives up to order $2m$, with $m \in \N$, $m \geq 1$. Then there exists a function $u \in \hh$ supported on $[0, T+1]$, such that $\cmp(t)=v_t$ for all $t\in[T]$ and
\begin{align*}
\|\Partial{u}\|_{L^1} \leq  C\|v_1\|_\ww+C \textstyle\sum_{t=2}^T \|v_t - v_{t-1}\|_\ww ,&\qquad
\norm{\cmp - \Partial[2m]{\cmp}}_{L^\infty}\le C'\max_t \norm{v_t}_\ww.
\end{align*}
with $C, C'$ depending only on $m$ and given in explicitly in the proof.
\end{restatable}
The theorem demonstrates that the continuous path-length can be bound by the usual discrete path-length under mild assumptions on the RKHS that are satisfied by the translation invariant kernel with spectral density $Q$ chosen according to \cref{thm:rkhs-to-pl}. Based on this observation, we immediately see that the RKHS characterized by the kernel in \Cref{prop:horizon-free-density}
satisfies the condition of the theorem, 
and has RKHS norm satisfying
$\norm{\cmp}_\hh^2=\widetilde{\cO}(\norm{\Partial{\cmp}}_{L^1}\norm{\cmp-\Partial[2]{\cmp}}_{L^\infty})=\widetilde{\cO}(M^2+M\sum_{t=2}^T\norm{\cmp_t-\cmp_\tmm}_\ww)$ where
$M=\max_t\norm{\cmp_t}_\ww$.

\paragraph{Optimal Path-length Dependencies. }
Applying our reduction \Cref{prop:pf-static-olo}
with the translation invariant kernel characterized
by \Cref{prop:horizon-free-density}, followed by
\Cref{thm:discrete-pl} to bound $\norm{\grad\cmp}_{L^{1}}=\cO(\sqrt{M^{2}+MP_{T}})$
immediately yields the following dynamic regret guarantee for OLO.

\begin{restatable}{theorem}{OptimalPL}\label{thm:optimal-pl}
Let $G>0$ and apply the algorithm characterized in \Cref{prop:pf-static-olo} with the kernel 
with spectral density described by \Cref{prop:horizon-free-density}.
Then for any $T>3$, and 
any sequence $\g_1,\ldots,\g_T$
satisfying $\norm{\gt}_\wwdual\le G$ and 
sequence $\cmp_1,\ldots,\cmp_T$ in $\ww\subseteq\R^d$, the dynamic regret is bounded as
\begin{align*}
R_T(\cmp_1,\ldots,\cmp_T)
&= 
\tilde{\mathcal{O}}\Big(\sqrt{(M^2 + MP_T)\textstyle\sumtT \norm{\gt}^2_{\wwdual}}\Big),
\end{align*}
where $M=\max_t\norm{\cmp_t}_\ww$ and $P_T=\sum_{t=2}^T\norm{\cmp_t-\cmp_\tmm}_\ww$.
\end{restatable}

As observed in \Cref{sec:translation-invariant}, the kernel that produces this
result is horizon independent, so the algorithm described above requires no
prior knowledge of $T$. This is in fact the first dynamic regret algorithm we
are aware of that achieves the optimal $\sqrt{P_T}$ dependence in the absence of prior
knowledge of $T$ without resorting
to a doubling trick. Likewise, in \Cref{app:scale-free} we show that these guarantees extend immediately to \emph{scale-free} guarantees using the gradient-clipping argument of \cite{cutkosky2019artificial}.
These are the first scale-free dynamic regret guarantees that we are aware of that achieve the optimal $\sqrt{P_T}$ dependencies.

\section{\SecCurvature}
\label{sec:curvature}

An advantage of our reduction over the dynamic-to-static
reduction of \citet{jacobsen2024equivalence} is that, by preserving the curvature
of the losses, our reduction
allows us to apply
(quasi)second-order methods like Online Newton Step (ONS)~\citep{HazanAK07}.

\paragraph{Exp-concave Losses}
The following proposition shows that
exp-concave losses retain the crucial property 
required to apply ONS under our reduction (proof in \Cref{app:curvature}).
\begin{restatable}{proposition}{LossCurvature}\label{prop:loss-curvature}
  Let $\ell_{t}:\ww\to\R$ be a $\beta$-exp-concave function,
  let $\hh$ be an RKHS with %
  feature map $\phi(t)\in \hh$,
  and define $\elltilde_{t}(\W)=\ell_{t}(\W\phi(t))$ for $\W\in\WW$.
  Then for any $X,Y\in\WW$,
  \begin{align*}
    \elltilde_{t}(X)-\elltilde_{t}(Y)\le \big\langle\grad\elltilde_{t}(X), X-Y\big\rangle_{\knorm}-\frac{\beta}{2}\big\langle\grad\elltilde(X),X-Y\big\rangle^{2}_{\knorm}~.
  \end{align*}
\end{restatable}
Note that this is precisely the curvature assumption that
is required to run Kernelized ONS (KONS)~\citep{calandriello2017second,calandriello2017efficient}.
Hence, applying our reduction \Cref{thm:RKHS-SR-solves-DR}
with KONS to the loss sequence $\elltilde_{1},\ldots,\elltilde_{T}$
leads immediately to the following dynamic regret guarantee, adapted from
\citet[Theorem 1]{calandriello2017second}.
\begin{restatable}{theorem}{KONS}\label{thm:kons}
  Let $\ell_{1},\ldots,\ell_{T}$ be a sequence of $\beta$-exp-concave losses.
  For
  any sequence $\cmp_{1},\ldots,\cmp_{T}\in\ww$ and
  $\Cmp\in\WW$ satisfying $\cmp_{t}=\Cmp\phi(t)$ for all $t$,
  \Cref{alg:kernel-oco} applied with KONS guarantees
  \begin{align*}
    R_{T}(\cmp_{1},\ldots,\cmp_{T})
    &=
      \mathcal{O}\Big(\lambda\norm{\Cmp}_{\knorm}^{2} + \deff\brac{\frac{\lambda}{\beta G^{2}\kmax}}\frac{\Log{2\beta G^{2}\kmax T}}{\beta}\Big),
  \end{align*}
  where $G\ge \norm{\grad\ell_{t}(\w)}$ for all $w\in\W$, $\kmax=\max_{t}k(t,t)$, 
  $\deff(\lambda)=\Tr{K_{T}(K_{T}+\lambda I)^{\inv}}$, and
  $K_{T}=(\inner{\grad\ell_{i}(w_{i}),\grad\ell_{j}(w_{j})}_\ww k(i,j))_{i,j=1}^{T}\in\R^{T\times T}$.
\end{restatable}
  In the previous section, we observed a direct trade-off
  between the complexity of the comparator---measured in terms of
  $\norm{\cmp}_{\hh}$---and a term measuring the complexity of the RKHS,
  $\max_{t}k(t,t)$. Here we again see a trade-off
  in measures of complexity, but now the complexity of the
  RKHS is characterized by the \emph{effective dimension}
  $\deff(\lambda)$. Loosely speaking,
  the effective dimension
  represents the number of ``non-negligable directions'' spanned by
  the features $\phi(1),\ldots,\phi(T)$,
  characterized by the number of eigenvectors of $K_{T}$ associated with
  non-negligable eigenvalues relative to $\lambda$.

  \paragraph{Strongly-convex Losses} Interestingly, for strongly-convex losses it can be shown
  that an analogous curvature condition to \Cref{prop:loss-curvature} holds
  under our reduction as well, leading to an analogous result to \Cref{thm:kons}.
  Indeed,
  the main difference is that in the strongly-convex setting,
  one uses the feature covariance $\lambda I +\sum_{s=1}^{t}\phi(t)\otimes\phi(t)$
  to define a weighted norm while KONS the covariance matrix
  of the product kernel associated with features $\tilde\phi(t)=\gt\otimes\phi(t)$.
  Applying a similar argument then leads to a guarantee which
  is analogous to \Cref{thm:kons}
  (see \Cref{app:sc} for more details).

\paragraph{Online Linear Regression} Similar results also apply in the context
of online regression. In that setting,
at the start of round $t$ the learner
first observes a \emph{context} $x_{t}\in\xx$,
then predicts a $\yhat_{t}\in\R$, and
incurs a loss $\ell_{t}(\yhat)=\half(\yt-\yhat)^{2}$.
In this setting, our reduction 
recovers \emph{kernelized online regression},
by letting $\yhat_{t}=\inner{f, \phi(x_{t})}$
where $f\in\Lin(\hh,\R)$ and $\phi(x_{t})\in\hh$ is the feature map
associated with $\hh$. Applying the Kernelized Vovk-Azoury-Warmuth
forecaster~\citep{AzouryW01, Vovk01,jezequel2019efficient} guarantees regret of the same form as
above. The result follows from \citet[Proposition 1 and Proposition 2]{jezequel2019efficient}.

\begin{restatable}{proposition}{RegressionCurvature}\label{prop:regression-curvature}
  Let $\ww=\R$ and for all $t$ let $\ell_{t}(\yhat)=\half(\yt-\yhat)^{2}$.
  Then for any sequence $(x_1,y_{1}),\dots,(x_T,y_{T})$ in $\xx\times\ww$ and
  any benchmark sequence $\cmpy_1,\ldots,\cmpy_T$ in $\R$ and $\cmp\in \hh$ satisfying $\cmpy_t=\inner{\cmp,\phi(x_{t})}$ for all $t$,
  the Kernelized VAW Forecaster guarantees
  \begin{align*}
    \sumtT (\ell_{t}(\yhat_{t})- \ell_{t}(\cmpy_t))
    &\le
      \lambda \norm{\cmp}^{2}_{\hh} + \deff(\lambda)y_{\max}^{2}\log\Big(e+\frac{eT\kmax^{2}}{\lambda}\Big),
  \end{align*}
  where $\kmax = \max_{t}k(t,t)$ and $y_{\max}^{2}=\max_{t}y_{t}^{2}$.
\end{restatable}
It is known that the dependence on $\deff(\lambda)$ for kernel ridge regression
is optimal \citep{lattimore2023lower}, demonstrating that these trade-offs are
unimprovable in the context of dynamic regret as well.

In each of the results above, the main trade-off is between the comparator norm
$\lambda\norm{\cmp}^{2}_{\hh}$ and the effective dimension $\deff(\lambda)$.
As an illustrative example, the following shows that the linear spline kernel
can achieve non-trivial \emph{squared} path-length guarantees, which were
recently
shown to be
unattainable in the
OLO setting \parencite{jacobsen2024equivalence}.
\begin{example}\label{ex:spline-kernel}
    The linear spline kernel $k(s,t)=\min(s,t)$
    has well-known RKHS norm of $\norm{\cmp}_\hh^2= \norm{\grad\cmp}_{L^2}^2=\int |\Partial{\cmp(t)}|^2 dt$. Moreover, in \Cref{app:sobolev} we show that  we can bound $\norm{\grad \cmp}_{L^2}\le \cO(\sqrt{\sum_t\abs{\cmpy_t-\cmpy_\tmm}^2_\ww}):=C_T'$  and 
    that the effective dimension is $\deff(\lambda)=\cO(T/\sqrt{\lambda})$ (\Cref{thm:discrete-squared-pl,thm:spline-deff} respectively).
    Optimally tuning $\lambda$ leads to
    $\DRegret(\cmpy_1,\ldots,\cmpy_T) = \tilde\cO(T^{2/3}(C_T')^{2/3})$
    which matches the minimax optimal rate
    for forecasting 
    in the class of discrete Sobolev sequences of bounded variation \citep{sadhanala2016total,BabyW19}.
    Note that
     $\lambda$ can be tuned without data-dependent prior knowledge using
     mixture-of-experts and a simple clipping argument \citep{mayo2022scalefree,jacobsen2024online}.
\end{example}
In the special case of the 1-dimensional squared loss $\ell_t(y)=\half(\yt-y)^2$, it is possible to achieve $R_T(\cmpy_1,\ldots,\cmpy_T) = \tilde\cO(T^{1/3}C_T^{2/3})$ where
$C_T= \sum_t\abs{\cmpy_t-\cmpy_\tmm}$ is the (unsquared) path-length of the benchmark \emph{predictions}, and this bound is
minimax optimal among the class of discrete TV-bounded sequences, which is more general than the Sobolev class in the example above \citep{BabyW19,zhang2025nonstationary}. Designing a kernel with a suitable effective dimension to achieve this trade-off has proven non-trivial and is left as a direction for future work.

\section{\SecDirectional}%
\label{sec:directional}

An exciting benefit of reducing to static regret
is that we can leverage more ``exotic'' static regret guarantees
to uncover new and interesting trade-offs in dynamic regret,
essentially for free.
For example, in recent years
there has been an interest in algorithms which adapt
to the \emph{directional} covariance between the
comparator and the losses~\citep{vanErvenK16,cutkosky2018black,cutkosky2020better,mhammedi2020lipschitz,cutkosky2024fully}, to guarantee
\begin{align*}
  R_{T}(\cmp)=\tilde{ \mathcal{O}}\Big(\sqrt{d\textstyle\sumtT \inner{\gt,\cmp}_{\ww}^{2}}\Big).
\end{align*}
These bounds recover the usual
$\tilde{\mathcal{O}}\big(\norm{\cmp}_{\ww}\sqrt{\sumtT\norm{\gt}^{2}_{\ww,*} }\big)$ bounds in the worst
case, but could be significantly smaller if the comparator tends to be
orthogonal to the losses.
Passing from dynamic regret to static regret via \Cref{thm:RKHS-SR-solves-DR},
the following proposition shows that guarantees of this form translate into dynamic regret guarantees
which naturally decouple the
comparator variability $\norm{\Cmp}_{\knorm}$ from the
a \emph{per-round} directional variance penalty $\sumtT\inner{\gt,\cmp_{t}}^{2}_{\ww}$. The full statement and proof of this result can be found in \Cref{app:directional}.
\begin{restatable}{proposition}{SimpleFullMatrix}\label{prop:simple-full-matrix}
  Let $\ell_{1},\ldots,\ell_{T}$ be an arbitrary sequence of $G$-Lipschitz convex
  loss functions over $\ww$.
  There exists an algorithm such that for any sequence of $\cmp_{1},\ldots,\cmp_{T}$ in $\ww$ and $\Cmp\in\WW$ satisfying $\cmp_{t}=\Cmp\phi(t)$ for all $t$, the dynamic regret $R_T(\cmp_1,\ldots,\cmp_T)$ is bounded by
  \begin{align*}
      \tilde{\mathcal{ O}}\Big(L_{k}\deff(\lambda) + \sqrt{\deff(\lambda)\sbrac{(\lambda + L^{2}_{k})\norm{\Cmp}^{2}_{\knorm}+\textstyle\sumtT \inner{\gt,\cmp_{t}}^{2}_\ww}\ln\Big(e+\frac{e \lambda_{\max}(K_{T})}{\lambda}\Big)}\Big),
  \end{align*}
  where $\gt\in\partial\ell_{t}(\wt)$,
  $L_{k}^{2} = G^{2} \max_{t}k(t,t)$, $K_{T}=(\inner{\gt,\gs}_\ww k(t,s))_{t,s\in[T]}$, and
  $d_{\mathrm{eff}}(\lambda)=\mathrm{Tr}(K_{T}(\lambda I + K_{T})^{\inv})$.
\end{restatable}

\section{Discussion}

In this paper we developed a general reduction from
dynamic regret to static regret based on embedding the
comparator sequence as a function in an RKHS.
We showed that
the optimal $\sqrt{P_{T}}$ path-length dependence of can be obtained
via a carefully designed translation-invariant kernel.
We also developed new scale-free and directionally-adaptive guarantees for online linear optimization and $\norm{\cmp}_\hh^2+\deff(\lambda)\ln T$ bounds for losses with curvature.

There are many promising directions for future work. As noted in \Cref{sec:olo}, if implemented naively, the algorithms described here could be prohibitively expensive to run in practice.
Future work should study how to best leverage kernel approximation techniques or sparse dictionary methods to achieve the standard $\mathcal{O}(d\ln T)$ per-round computation without ruining the desired regret bounds.
We also anticipate many interesting directions for future work
by investigating the rich intersections between
online learning, kernel methods, and signal processing that
our reduction brings to light.

\begin{ack}
AJ and NCB acknowledge the financial support from
the FAIR project, funded by the NextGenerationEU program within the PNRR-PE-AI scheme (M4C2, investment 1.3, line on Artificial Intelligence),
the MUR PRIN grant 2022EKNE5K (Learning in Markets and Society), funded by the NextGenerationEU program within the PNRR scheme (M4C2, investment 1.1),
the EU Horizon CL4-2022-HUMAN-02 research and innovation action under grant agreement 101120237, project ELIAS (European Lighthouse of AI for Sustainability),
and the One Health Action Hub, University Task Force for the resilience of
territorial ecosystems, funded by Università degli Studi di Milano (PSR
2021-GSA-Linea 6). AR acknowledges support from the European Research Council (grant REAL 947908).
\end{ack}

\bibliographystyle{plainnat_nourl}
\bibliography{refs_min}

\clearpage

\appendix

\section{Additional Functional Analysis Background}%
\label{app:bg}

In this section we briefly recall some additional definitions
and background from functional analysis, which will be
useful for understanding the proofs of our results but were not
relevant for the main text. 

\paragraph{Additional Notations.} The
\emph{dual space} of a normed space $\hh$ is the space of bounded linear functionals
$\hh^{*}=\Lin(\hh,\R)$, and the associated
norm is the \emph{dual norm} $\norm{\g}_{\hh,*}=\sup_{\norm{h}_{\hh}=1}g(h)$
for any $g\in \hh^{*}$.
An operator $\cT:\hh\to\ww$ is
\emph{Hilbert-Schmidt} if for any orthonormal basis $\Set{h_{i}}_{i}$
of $\hh$ we have $\norm{T}_{\HS}^{2}:=\sum_{i}\norm{Th_{i}}^{2}_{\ww}<\infty$.
The space $\Lin(\hh,\ww)$ is itself a Hilbert space
with inner product $\inner{A, B}_{\HS}=\sum_{i}\inner{Ah_{i},Bh_{i}}_{\ww}$.

\paragraph{A Brief Review of Reproducing Kernel Hilbert Spaces. }

A linear functional $\varphi\in \Lin(\hh,\R)$ is \emph{bounded}
if there exists a constant $M$ such that $\abs{\varphi(x)}\le M\norm{x}_{\hh}$ for all
$h\in\hh$.
A \emph{reproducing kernel Hilbert space} (RKHS) is
a Hilbert space $\hh$
of functions $h:X\to \R$ for which
the \emph{evaluation functional} $\delta_{x}:h \mapsto h(x)$ is bounded for all
$x\in X$.
For any such space, Riesz representation theorem tells us that for any $x\in X$
there is a unique $k_{x}\in \hh$ such that $\delta_{x}(h)=\inner{h,k_{x}}_{\hh}$
for all $h\in \hh$.
The function $k(x,x')=k_{x}(x')$ is called the
\emph{reproducing kernel} associated with $\hh$.
The reproducing kernel is often expressed in terms of
the \emph{feature map} $\phi(x)=k_{x}\in\hh$ as
$k(x,x')=\inner{\phi(x),\phi(x')}_{\hh}$.

We will often be interested in functions taking
values in $\ww\subseteq\R^{d}$. In this case the preceding discussion
can be extended in a straightforward way
by considering a coordinate-wise extension.
In particular, observe that in this setting an operator $\W\in\WW$ can be
represented as a tuple $(\W_{1},\ldots,\W_{d})$ such that
$\W_{i}\in\Lin(\hh,\R)=\hh^{*}$ for each $i\in[d]$.
Riesz representation theorem then tells us that
there is a $\w_{i}\in \hh$ such that $\W_{i}(h)=\inner{\w_{i},h}_{\hh}$
for any $h\in\hh$, so using the reproducing property we have $\W_{i}(\phi(t))=\inner{\w_{i},\phi(t)}_{\hh}=\w_{i}(t)$.
Hence, each $\W\in\Lin(\hh,\ww)$ is identified by
a tuple $(\W_{1},\ldots,\W_{d})\in\hh^{d}$ and we can write
\begin{align*}
  \w(t)
  =\big(\inner{\w_{1},\phi(t)}_{\hh},\ldots,\inner{\w_{d},\phi(t)}_{\hh}\big)
  =\big(\W_{1}(\phi(t)),\ldots, \W_{d}(\phi(t))\big) = \W\phi(t)\in\ww.
\end{align*}
Note the that space $\Lin(\hh,\ww)$ is itself a Hilbert space when
equipped with the Hilbert-Schmidt norm,
which in the coordinate-wise extension above
can be expressed as
$\norm{\W}_{\knorm}^{2}=\norm{\w}^{2}_{\hh^{d}}=\sum_{i=1}^{d}\norm{\w_{i}}_{\hh}^{2}$.
This can be seen by definition of
the Hilbert-Schmidt norm: let $\Set{h_{i}}_{i}$ be an orthonormal basis of $\hh$,
then
\begin{align*}
  \norm{\W}_{\knorm}^{2}
  &=
    \sum_{i}\norm{\W h_{i}}_{\ww}^{2}
    =
    \sum_{i}\norm{\brac{\W_{1}(h_{i}),\ldots,\W_{d}(h_{i})}}_{\ww}^{2}\\
  &=
    \sum_{i}\norm{\brac{\inner{\w_{1},h_{i}}_{\hh},\ldots,\inner{w_{d},h_{i}}_{\hh}}}^{2}_{\ww}\\
  &=
    \sum_{i}\sum_{j=1}^{d}\inner{\w_{j},h_{i}}^{2}_{\hh}
    =
    \sum_{j=1}^{d}\sum_{i}\inner{\w_{j},h_{i}}^{2}_{\hh} \\
  &=
    \sum_{j=1}^{d}\norm{\w_{j}}^{2}_{\hh}:=\norm{\w}_{\hh^{d}}^{2},
\end{align*}
where the last line uses Parseval's identity.

\section{Recovering \texorpdfstring{\textcite{jacobsen2024equivalence}}{Jacobsen \&
  Orabona (2024)}}
\label{app:recovering}

  In this section we demonstrate that the reduction in \textcite{jacobsen2024equivalence}
  is equivalent to the special case of our framework.
  Note that we assume linear losses in this section
  because the reduction of \textcite{jacobsen2024equivalence} is only
  defined for linear losses (or by linearizing the losses $\ell_{t}$ via convexity).

  Let $\et\in\R^{T}$ be the $t^{\textrm{th}}$ standard basis vector
  and consider \Cref{alg:kernel-oco} with
  $\hh=\R^{T}$, $\inner{A,B}_{\HS}=\Tr{A^{\top}B}$, kernel
  feature map
  $\phi(t)=\et\in\R^{T}$, and linear losses
  $\W\mapsto \inner{\Gt,\W}_{\HS}$ for $\Gt=\gt\otimes \et=\gt\et^{\top}\in\R^{d\times T}$.
  Then for any sequence $\cmp_{1},\ldots,\cmp_{T}$ in
  $\R^{d}$, let $\Cmp=\pmat{\cmp_{1}&\dots&\cmp_{T}}\in\R^{d\times T}$
  and observe that we can write
  \(
  \cmp_{t} =
  \Cmp\phi(t).
  \)
  Moreover,
  \begin{align*}
    \sumtT \inner{\gt,\cmp_{t}}
    &=
      \sumtT\inner{\gt, \Cmp\phi(t)}
      =
      \sumtT \Tr{\phi(t)\gt^{\top}\Cmp}
      =
      \sumtT \Tr{\brac{\gt\phi(t)^{\top}}^{\top}\Cmp}\\
    &=
      \sumtT \inner{\gt\phi(t)^{\top}, \Cmp}_{\HS}
      =
      \sumtT \inner{\Gt,\Cmp}_{\HS}.
  \end{align*}
  Similarly, suppose
  $\cA$ is an online learning algorithm  and let
  $\Wt=\pmat{\wt^{(1)}&\dots&\wt^{(T)}}\in\R^{d\times T}$ denote its output on
  round $t$. Suppose on round $t$ we play $\wt=\Wt\phi(t)=\wt^{(t)}$.
  Then
  \begin{align*}
    \sumtT \inner{\gt,\wt}
    &=
      \sumtT \inner{\gt, \Wt\phi(t)}
      =
      \Tr{\phi_{t}\gt^{\top}\Wt}=\inner{\gt\phi(t)^{\top},\Wt}_{\HS} = \sumtT \inner{\Gt,\Wt}_{\HS}.
  \end{align*}
  Thus,
  \begin{align*}
    \DRegret(\cmp_{1},\ldots,\cmp_{T})
    &=
      \sumtT \inner{\gt,\wt-\cmp_{t}}
    =
      \sumtT \inner{\Gt, \Wt - \Cmp}_{\HS} = \RKHSRegret[\cA](\Cmp).
  \end{align*}
  To see why this is precisely equivalent to the
  reduction of \textcite{jacobsen2024equivalence},
  observe that their reduction is simply phrased in terms
  of the ``flattened'' versions of each of the above
  quantities, and can be interpreted as working in
  the finite-dimensnional RKHS over $\tilde{\cH}=\R^{dT}$
  with $\inner{x,y}_{\tilde{\cH}}=\inner{x,y}$ being the canonical inner
  product on $\R^{dT}$.
  In particular, they instead define
  \begin{align*}
    \tilde{\cmp}=\Vec{\Cmp}=\!\!\pmat{\cmp_{1}\\\vdots\\\cmp_{T}}\in\R^{dT},\>
    \tilde{\wvec}_{t}=\Vec{\Wt}=\!\!\pmat{\wt^{(1)}\\\vdots\\\wt^{(T)}}\!\!\in\R^{dT},\>
    \tilde{\gvec}_{t}= \Vec{\Gt}=\!\!\pmat{\zeros\\\vdots\\\gt\\\zeros\\\vdots}\!\!\in\R^{dT},
  \end{align*}
  and run an algorithm $\tilde{\cA}$ defined on $\tilde \cH=\R^{dT}$
  against the losses $\tilde{\gvec}_{t}$.
  As shown in their Proposition 1, under this setup it holds that
  $\sumtT \inner{\gt,\wt-\cmp_{t}}=\sumtT \inner{\gtilde_{t},\tilde{\w}_{t}-\tilde{\cmp}}$,
  from which it immediately follows that
  \begin{align*}
    R_{T}(\cmp_{1},\ldots,\cmp_{T})
    &=
      \overbrace{\sumtT \inner{\Gt,\Wt-\Cmp}_{\HS} =\RKHSRegret[\cA](\Cmp)}^{\text{\Cref{thm:RKHS-SR-solves-DR}}}\\
    &=
      \underbrace{\sumtT \inner{\tilde{\gvec}_{t},\tilde{\wvec}_{t}-\tilde{\cmp}}_{\tilde{\cH}}
      = \RKHSRegret[\tilde{\cA}](\tilde{\cmp}).}_{\text{\citet[Proposition 1]{jacobsen2024equivalence}}}
  \end{align*}


\newpage
\section{Proofs for Section~\ref{sec:olo} (\SecOLO)}%
\label{app:olo}

\subsection{Proof of Theorem~\ref{thm:olo-reduction}}%
\label{app:olo-reduction}
\OLOReduction*
\begin{proof}
We note that $L(\hh, \ww)$ is a separable Hilbert
space for separable $\hh$ and $\ww$. Moreover, $\tilde{\ell}_t$ is differentiable, with derivative
\[
\nabla \tilde{\ell}_t(\W)
= \nabla \ell_t(\W \phi(t)) \otimes \phi(t) \in L(\hh, \ww),
\]
where $\otimes$ is the tensor product. Then applying algorithm $\mathcal{A}$ to
the loss sequence $(\tilde{\ell}_t)_{t}$, we obtain
\[
\RKHSRegret(\Cmp)
\leq \phi(\|\Cmp\|_{L(\hh, \ww)}, \|\nabla \ell_{1}(\wvec_{1}) \otimes \phi(1)\|_{L(\hh, \ww)},\dots,\|\nabla \ell_t(\wvec_{T}) \otimes \phi(T)\|_{L(\hh, \ww)})~.
\]
The proof is concluded by noting that $\DRegret(\cmp_{1},\ldots,\cmp_{T}) = \RKHSRegret(\Cmp)$ by \cref{thm:RKHS-SR-solves-DR}, and that
\[
\|\nabla \ell_t(\wt) \otimes \phi(t)\|_{L(\hh, \ww)}
= \|\nabla \ell_t(\wt)\|_{\ww} \, \|\phi(t)\|_{\hh},
\]
since $u \otimes v \in L(\cU, \cV)$ is a rank one operator and so $\|u \otimes v\|_{L(\cU,\cV)} = \|u\|_{\cU}\|v\|_{\cV}$, for any $u,v \in \cU, \cV$ and $\cU, \cV$ Hilbert spaces. Finally, note that
\(
\|\phi(t)\|^2_\hh
= \scal{\phi(t)}{\phi(t)}
= k(t,t)~. \qedhere
\)
\end{proof}
\begin{remark}\label{remark:dual-norms}
    Note that the result
    of \Cref{thm:olo-reduction} applies more generally to algorithms that use dual-weighted-norm pairs $(\norm{\cdot}_M, \norm{\cdot}_{M^{\inv}})$, since this amounts
    to transforming the decision space $\W\mapsto M^\half \W$,
    the losses $\Gt\mapsto M^{-\half}\Gt$, and preserving the original inner product structure. We expect the theorem should
    also generalize to arbitrary dual-norm pairs on $\ww$, but this will require
    some additional care to interpret the norm
    $\norm{\cdot}_{\ww^{*}\otimes\hh}$.
\end{remark}

\subsection{A Concrete Example of Proposition~\ref{prop:pf-static-olo}}%
\label{app:pf-static-olo}

\SetAlCapHSkip{0.5em}\setlength{\algomargin}{0.5em}
\begin{algorithm}
  \SetAlgoLined
  \textbf{Input: }Lipschitz bound $G\ge \norm{\gt}_{\ww}$ for all $t$, Value $\epsilon>0$\\
  \textbf{Initialize: }{$\w_{1}=\zeros$, $G_{0}=G\max_{t}\sqrt{k(t,t)}$,
    $V_{1}=4G_{0}^{2}$, $S_{1}=0$}\\
\textbf{Define: }
    \(
      \Psi(S,V)=\begin{cases}
                     \frac{\epsilon G_{0}}{\sqrt{V}\log^{2}(V/G_{0}^{2})}\sbrac{\Exp{\frac{S^{2}}{36 V}}-1}&\text{if }S\le \frac{6V}{G_{0}}\\
                     \frac{\epsilon G_{0}}{\sqrt{V}\log^{2}(V/G_{0}^{2})}\sbrac{\Exp{\frac{S}{3G_{0}}-\frac{6V}{G_{0}}}-1}&\text{otherwise}
                \end{cases}
    \)\\
  \For{$t=1:T$}{
    Play $\wt$, receive subgradient $g_{t}$\\
    Set $V_{\tpp}=V_{t}+\norm{g_{t}}^{2}_{\ww}k(t,t)$\\
    Set
    $S_{\tpp}^{2}= S_{t}^{2}+k(t,t)\norm{\gt}^{2}_{\ww}+2\sum_{s=1}^{\tmm}k(s,t)\inner{\g_{t},\g_{s}}_{\ww}$\\
  \BlankLine
    Update
    $\wtpp = \frac{-\sum_{s=1}^{t}k(s,\tpp)\gs}{S_{\tpp}}\Psi(S_{\tpp},V_{\tpp})$
  }
  \caption{Kernelized Instance of \citet[Algorithm 4]{jacobsen2022parameter}}
  \label{alg:kernelized-pf}
\end{algorithm}

There are many examples of algorithms which would produce the
static regret guarantee stated in \Cref{prop:pf-static-olo}. In this section, we
briefly provide an example which attains the result of the stated form, and
provide the full regret guarantee and update in our framework.

Let us consider the algorithm characterized by \citet[Theorem
1]{jacobsen2022parameter} in an unconstrained setting.
Their algorithm
can be understood as a particular instance of FTRL,
and so we can develop its kernelized version
using the same reasoning as \Cref{ex:kernel-ftrl}.
Indeed, applying their algorithm in the space
$\WW$ with inner product $\inner{\cdot,\cdot}_{\knorm}$
against losses $\Gt=\gt\otimes\phi(t)$ leads to
updates of the form
\begin{align*}
  \Wtpp = \frac{-\sum_{s=1}^{t}\Gs}{\norm{\sum_{s=1}^{t}\Gs}_{\knorm}}\Psi\brac{\norm{\sum_{s=1}^{t}\Gs}_{\knorm}, V_{\tpp}},
\end{align*}
where $V_{\tpp}=4G_{0}^{2}+\sum_{s=1}^{t}\norm{\Gs}^{2}_{\knorm}$ and $\Psi(S,V)$ defined in \Cref{alg:kernelized-pf}.
Moreover, we have
$V_{\tpp}=4G_{0}^{2}+\sum_{s=1}^{t}\norm{\gs}_{\ww}^{2}k(s,s)$
and $\norm{\sum_{s=1}^{t}\Gs}_{\knorm}=\sum_{i,j}^{t}\inner{g_{i},g_{j}}k(i,j)$
via \Cref{lemma:hs-norm} and \Cref{lemma:hs-sum-norm} respectively,
and so in the context of our reduction \Cref{alg:kernel-oco}
the updates are
\begin{align*}
  \wtpp&=\Wtpp\phi(\tpp)\\
       &=
         \frac{-\sum_{s=1}^{t}\Gs\phi(\tpp)}{\norm{\sum_{s=1}^{t}\Gs}_{\knorm}}\Psi\brac{\norm{\sum_{s=1}^{t}\Gs}_{\knorm}, V_{\tpp}}\\
       &=
         \frac{-\sum_{s=1}^{t}k(s,t)\gs}{\sqrt{\sum_{i,j=1}^{t}k(i,j)\inner{\g_{i},\g_{j}}_{\ww}}}\Psi\brac{\sqrt{\sum_{i,j=1}^{t}k(i,j)\inner{\g_{i},\g_{j}}_{\ww}}, 4G_{0}^{2}+\sum_{s=1}^{t}k(s,s)\norm{\gs}^{2}_{\ww}},
\end{align*}
leading to the procedure described in \cref{alg:kernelized-pf}.
Notice that, as mentioned in \Cref{sec:olo}, this naive implementation requires
$\cO(t)$ time and memory to update on round $t$ due to having to re-weight the sum
$\sum_{s=1}^{t}k(s,t)\gs$ and compute $\sum_{s=1}^{\tmm}k(s,t)\inner{\gt,\gs}_{\ww}$,
so in practice one would ideally implement additional measures to reduce the
complexity, such as implementing Nystrom projections or choosing a suitably
sparse kernel.

Now applying \Cref{alg:kernelized-pf} with
\Cref{thm:olo-reduction}, we immediately get the
following regret guarantee from
\citet[Theorem 1]{jacobsen2022parameter}.
\begin{restatable}{proposition}{KernelizedPF}\label{prop:kernelized-pf}
  Let $\ell_{1},\ldots,\ell_{T}$ be $G$-Lipschitz convex loss functions and let
  $\gt\in\partial\ell_{t}(\wt)$ for all $t$. For any $\cmp_{1},\ldots,\cmp_{T}$ in
  $\ww$ and $\Cmp\in\WW$ satisfying $\cmp_{t}=\Cmp\phi(t)$ for all $t$,
  \Cref{alg:kernelized-pf} guarantees
  \begin{align*}
    R_{T}(\cmp_{1},\ldots,\cmp_{T})
    &=
      \RKHSRegret(\Cmp)\\
    &\le
      4G_{0}\epsilon + 6\norm{\Cmp}_{\knorm}\Max{\sqrt{V_{T+1}\Log{\frac{\norm{\Cmp}_{\knorm}}{\alpha_{T+1}}+1}},G_{0}\Log{\frac{\norm{\Cmp}_{\knorm}}{\alpha_{T+1}}+1}},
  \end{align*}
  where $V_{T+1}=4G_{0}^{2}+\sumtT\norm{\gt}^{2}_{\ww}k(t,t)$ and
  $\alpha_{T+1}=\frac{\epsilon G_{0}}{\sqrt{V_{T+1}}\log^{2}(V_{T+1}/G_{0}^{2})}$.
\end{restatable}

\subsection{Scale-free Guarantees}%
\label{app:scale-free}
Now that we have seen how to obtain the optimal path-length
dependencies on Lipschitz losses, we can extend these
guarantees to be \emph{scale-free} by simply changing the
base algorithm.
In particular, there are algorithms which
are adaptive to both the comparator norm and the
effective Lipschitz constant, $\Lip_{T}=\max_{t\in[T]}\norm{\grad\ell_{t}(\wt)}_{\ww}$.
Algorithms
which scale with $\Lip_{T}$ rather than a
given
upper bound $G\ge \Lip_{T}$ are referred to as
\emph{scale-free}. We first consider the setting
in which the domain is constrained
$\ww=\Set{\w\in\R^{d}:\norm{\w}_{\ww}\le D}$.\footnote{More generally,
  this assumption amounts to assuming prior knowledge
  on a bound $D\ge \norm{\cmp_{t}}_{\ww}$ for all $t$, which the
  learner can leverage by projecting to the same set, regardless of
  any boundedness of the underlying problem's domain.
}
\begin{restatable}{proposition}{ScaleFreeOLO}\label{prop:scale-free-olo}
  There exists an algorithm
  which guarantees that for any sequence $\cmp_{1},\dots,\cmp_{T}$
  in $\ww=\Set{\w\in\R^{d}:\norm{\w}_{\ww}\le D}$,
  \begin{align*}
    \DRegret(\cmp_{1},\dots,\cmp_{T})
    &=
      \tilde O\brac{\Lip_{T}(\max_{t}\norm{\cmp_{t}}_{\ww}+D) + \norm{\Cmp}_{\HS}\sqrt{\Lip_{T}^{2}+\sumtT \norm{\gt}^{2}_{\ww}k(t,t)}}~,
  \end{align*}
  where $\Lip_{t}=\max_{t}\norm{\gt}_{\ww}$.
\end{restatable}
The result follows by constraining $\norm{\wt}_{\ww}\le D$
and applying the gradient-clipping
argument of \citet{cutkosky2019artificial}.\footnote{
  Note that constraining the final outputs
  $\wt$ is straight-forward in our framework;
  one can simply apply a standard unconstrained-to-constrained
  reduction in $\ww$ \citep{cutkosky2018black,cutkosky2020parameter}
  prior to applying our dynamic-to-static reduction.
}
Indeed, if we've constrained our iterates to satisfy $\norm{\wt}_{\ww}\le D$,
then we can replace the gradients $\gt$ the the clipped
gradients $\hat \gt=\gt\Min{1, \frac{\max_{s<t}\norm{g_{s}}_{\ww}}{\norm{g_{t}}_{\ww}}}$
to get
\begin{align*}
  R_{T}(\cmp_{1},\ldots,\cmp_{T})&=
  \sumtT \inner{\hat\gt,\wt-\cmp_{t}}_{\ww} + \sumtT \inner{\gt-\hat\gt,\wt-\cmp_{t}}_{\ww}\\
  &\le
    \hat R_{T}(\cmp_{1},\ldots,\cmp_{T}) + \max_{t}\norm{\gt}_{\ww}(D+\max_{t}\norm{\cmp_{t}}_{\ww})
\end{align*}
following the same telescoping argument as \citet{cutkosky2019artificial}.
With this in hand, we can simply apply our reduction
\Cref{alg:kernel-oco} with the losses $\hat\Gt = \hat\gt\otimes\phi(t)$,
which we now have an \emph{a priori} bound on at the
start of round $t$:
$\norm{\hat\Gt}_{\knorm}\le \max_{s<t}\norm{\gs}_{\ww}\sqrt{k(t,t)}:=\hat\Lip_t\le \Lip_{t}$.
Hence, even without prior knowledge of a $G\ge \norm{\gt}_{\ww}$ we can obtain
$\hat R_{T}(\cmp_{1},\ldots,\cmp_{T})\le \tilde O\brac{\norm{\Cmp}_{\knorm}\sqrt{\max_{t}\norm{\gt}^{2}_{\ww}k(t,t)+\sumtT \norm{\gt}^{2}_{\ww}k(t,t)}}$.

To see why this is difficult using existing techniques,
note that nearly
all existing algorithms which achieve the optimal $\sqrt{P_{T}}$ dependence
do so by designing an algorithm which guarantees
dynamic regret of the form
\begin{align*}
\DRegret(\cmp_{1},\dots,\cmp_{T})= \tilde{\mathcal{ O}}\brac{\frac{P_{T}}{\eta}+\eta G^{2}T},
\end{align*}
from which $G\sqrt{P_{T}T}$ regret is obtained by tuning $\eta$.\footnote{The exception being
  \textcite{zhang2023unconstrained}, which uses a similar high-dimensional
  embedding as \cite{jacobsen2024equivalence}, but neither works obtain the
  optimal $\sqrt{P_{T}}$ while being scale-free.
} The tuning step
is done by running several instances of the base algorithm in
parallel for each $\eta$ in some set
$\cS=\Set{2^{i}/G\sqrt{T}\minOp 1/G: i=0,1,\dots}$, and combining the
outputs---typically using a mixture-of-experts algorithm like Hedge.
Note however that the set $\cS$ requires prior knowledge of
the Lipschitz constant $G$. There is no straightforward way to
adapt to this argument without resorting to unsatisfying doubling
strategies, which are well-known to perform poorly in practice.
Instead, using our framework we avoid these issues entirely
by simply applying a scale-free static regret guarantee
to get the a $\Lip_{T}\norm{\Cmp}_{\knorm}\sqrt{T}$ dependence,
and then designing a kernel which ensures
$\norm{\Cmp}_{\knorm}\le\sqrt{MP_{T}}$.

More generally, when the domain $\ww$ is not uniformly bounded,
it is still possible to achieve a scale-free bound
at the expense of an $\Lip_{T}\max_{t}\norm{\cmp_{t}}^{3}_{\ww}$ penalty,
again using the same argument as \cite{cutkosky2019artificial}.
One simply starts by replacing the comparator sequence $\cmp_{1},\ldots,\cmp_{T}$
with a new one satisfying $\hat\cmp_{t}=\Pi_{\ww_{t}}\cmp_{t}$,
where $\ww_{t}=\Set{\w\in\ww: \norm{\w}_{\ww}\le \sqrt{\sum_{s=1}^{\tmm}\norm{\gs}_{\ww}}}$.
Then one can show that
\begin{align*}
  R_{T}(\cmp_{1},\ldots,\cmp_{T})
  &=
    \sumtT \inner{\gt,\wt-\cmp_{t}}_{\ww}
    =
    \sumtT \inner{\gt,\wt-\hat\cmp_{t}}_{\ww}+\sumtT \inner{\gt,\hat\cmp_{t}-\cmp_{t}}_{\ww}\\
  &\le
   \cO\brac{R_{T}(\hat\cmp_{1},\ldots,\hat\cmp_{T}) +G\max_{t}\norm{\cmp_{t}}^{3}_{\ww}}.
\end{align*}
Applying the same clipping argument as above and observing that $\hat P_T=\sum_t\norm{\hat\cmp_t-\hat\cmp_\tmm}_\ww\le \cO(P_T + \max_t\norm{\cmp_t}_{\ww})$ we get the following result.
\begin{restatable}{proposition}{UnconstrainedScaleFree}\label{prop:unconstrained-scale-free}
  There exists an algorithm such that for any $\cmp_{1},\dots,\cmp_{T}$
  in $\ww\subseteq\R^{d}$,
  \begin{align*}
    R_{T}(\cmp_{1},\dots,\cmp_{T})
    &=
      \tilde{\mathcal{O}}\brac{
      \Lip_{T}\max_{t}\norm{\cmp_{t}}^{3}_{\ww} + \norm{\Cmp}_{\knorm}\sqrt{\sumtT \norm{\gt}_{\ww}^{2}k(t,t)}
      }.
  \end{align*}
  where $\Lip_{T}=\max_{t}\norm{\gt}_{\ww}\sqrt{k(t,t)}$.
\end{restatable}


\newpage
  \section{Intuitions on the Spectral Density \texorpdfstring{$Q(\omega)$}{Q(w)}}%
  \label{app:spectral-intuition}

  In this section we provide some additional high-level intuitions that motivate
  the choice of $Q(\omega)$ in \Cref{prop:horizon-free-density}.

   Recall that we would like to choose $Q(\omega)\approx 1/\abs{\omega}$, since
   this choice yields a continuous-time analogue of the
   $M P_{T}$ dependence that we want to achieve:
   $\norm{\cmp}^2_\hh \le \max_{t}\abs{\cmp(t)}\norm{\grad\cmp}_{L^{1}}$.
   The key issue is that this choice of $Q(\omega)$ is
   not integrable, diverging as $\omega\to 0$ and $\omega\to \infty$. We fix
   these issues by adding a small amount of additional regularization, setting
   \begin{align*}
     Q(\omega) \propto \frac{R_0(\omega) R_\infty(\omega)}{\abs{\omega}}~,
   \end{align*}
   where $R_\infty(\omega)\approx 1/(1+\abs{\omega}^2)^{1/4}$ is a \emph{tapering function} that ensures integrability in the asymptotic regime, and $R_0(\omega) \approx 1/\log(1+1/\sqrt{\abs{\omega}})\log^2(\log(1+1/\sqrt{\abs{\omega}}))$ ensures that $Q$ is well-behaved near zero.
   It is clear that $R_{\infty}(\omega)$ ensures integrability in the asymptotic
   regime since when $\omega$ is large we have
   $Q(\omega)\approx R_{\infty}(\omega)/\abs{\omega} \le 1/\abs{\omega}^{2}$,
   which is integrable away from zero.
   On the other hand, near zero $R_{\infty}(\omega)\approx1$ and we have
   \begin{align*}
     Q(\omega)\approx \frac{R_{0}(\omega)}{\abs{\omega}} \approx \frac{1}{\abs{\omega}\log(\abs{\omega}^{-1/2})\log^{2}\log(\abs{\omega}^{-1/2})}~,
   \end{align*}
   which after a change of variables $t=\log(\omega^{-1/2})$ integrates near zero
   as
   \begin{align*}
     \int_{-\epsilon}^{\epsilon}Q(\omega)d\omega=2\int_{0}^{\epsilon}Q(\omega)d\omega \approx \int_{\log(1/\epsilon)}^{\infty}\frac{1}{t\log^{2}(t)}dt= \cO(1)
   \end{align*}
   for an appropriately chosen $\epsilon$.
\section{Proofs for Section~\ref{sec:optimal-pl} (\SecOptimalPL)}
\label{app:optimal-pl}

\subsection{Proof of Theorem~\ref{thm:rkhs-to-pl}}%
\label{app:rkhs-to-pl}
\RKHSToPL*
\begin{proof}
Consider a function $f$ that is supported on $[0,\tau]$ for some $\tau>0$.
To bound the norm above in terms of the $L^1$ norm we use the fact that for any $u,v,w \in L^2(\R) \oplus L^1(\R)$ we have $F[ \Partial[k]{u}] = (2 \pi i)^k \widehat{u}$, for $k \in \N$,  $F[u \star v] = \widehat{u} \widehat{v}$, where $\star$ is the convolution operator, and that by the Plancherel theorem we have $\int_\R \widehat{u}(\omega) \widehat{v}(\omega) dt = \int_\R u(t) v(t) dt$ and so, in particular, $\int_\R \widehat{u}(\omega) \widehat{v}(\omega) \widehat{w}(\omega) dt = \int_\R u(t)  (v \star w)(t) dt$. Moreover, note that, by construction $R$ is an integrable real odd function, so its Fourier transform $\widehat{R}$ is an odd purely imaginary function. So, we have
\begin{align*}
\|f\|^2_\hh &:= \int_\R \frac{|\widehat{f}(\omega)|^2}{Q(\omega)} = i\int_{\R} \overline{\frac{\omega \widehat{f}(\omega)}{2 \pi i}} \; \left(1+\frac{\omega^{2m}}{(2 \pi)^{2m}}\right) \widehat{f}(\omega) \; R(\omega) d \omega \\
& = i\int_{\R} \overline{\Partial{f}(-t)} \; ((f - \Partial[2m]{f}) \star \widehat{R})(t) d t \\
& =i \int_{-\tau}^0 \overline{\Partial{f}(-t)} \; ((f - \Partial[2m]{f}) \star \widehat{R})(t) d t \\
&\leq \|\Partial{f}\|_{L^1} \|(f - \Partial[2m]{f}) \star \widehat{R}\|_{L^\infty([-\tau,0])},
\end{align*}
where in the last two steps we used the fact that $f$ is bounded in $[0, \tau]$ and the Hölder inequality. Since also $\Partial[m]{f}$ is supported on $[0,\tau]$, we have
\begin{align*}
\|(f - \Partial[2m]{f}) \star \widehat{R}\|_{L^\infty([-\tau,0])} &= \sup_{t' \in [-\tau, 0]} \int_{0}^\tau |f(t) - \Partial[2m]{f}(t)||\hat R(t'-t)| dt \\ & \leq \|f - \Partial[2m]{f}\|_{L^\infty(\R)} \|\widehat{R}\|_{L^1([-\tau,\tau])}~.
\end{align*}
The last step is finding a bound that is easy to compute for $c(\tau):=\|\widehat{R}\|_{L^1([-\tau,\tau])}$. We start noting that since $R$ is odd
$$
\widehat{R}(t) := \int_\R R(x) e^{2 \pi i x t} dx = 2i \int_0^\infty R(x) \sin(2 \pi i x t) dx,
$$
and, in particular also $\widehat{R}$ is an odd function. The characterization
we propose for $c(\tau)$ is a direct consequence of the fact that the sine
transform of $R$ is non-negative (also known as Polya criterion, which we recall
in \cref{lemma:bound-R} and is applicable since $R$ is positive and decreasing on $(0,\infty)$). Now, by expanding the definition of $\widehat{R}$ and using the non-negativity of the sine transform, we have
\begin{align*}
\int_{-\tau}^\tau |\widehat{R}(t)| dt &= 2 \int_{0}^\tau |\widehat{R}(t)| dt = 2\int_{0}^\tau \left| \int_0^\infty R(x) \sin(2\pi x t) dx \right| dt \\
& = 2 \int_0^\tau \int_0^\infty R(x) \sin(2\pi x t) dx dt = 2\int_0^\infty R(x) \left(\int_0^\tau  \sin(2\pi x t) dt\right) dx \\
& = 2\int_0^\infty R(x) \frac{\sin(\tau \pi x)^2}{\pi x}dx.
\end{align*}
To conclude, let $\alpha > 0$, since $|\sin(z)| \leq \min(|z|,1)$ for any $z \in \R$
\begin{align*}
\int_0^\infty R(x) \frac{\sin(\tau \pi x)^2}{\pi x}dx & = \int_0^\alpha R(x) \frac{\sin^2(\tau \pi x)}{\pi x} dx + \int_\alpha^\infty R(x) \frac{\sin^2(\tau \pi x)}{\pi x} dx \\
& \leq \pi \tau^2 \int_0^\alpha R(x) x dx + \frac{1}{\pi} \int_\alpha^\infty \frac{R(x)}{x} dx.
\end{align*}
The stated result then follows by choosing $\tau=T+1$.
\end{proof}

\subsection{Proof of Theorem~\ref{thm:discrete-pl}}%

Before proving \cref{thm:discrete-pl} we need two auxiliary results

\begin{lemma}\label{lm:construction-of-bT}
Given $m \in \N$ with $m \geq 1$ and $T > 0$ there exists a function $b_T$ that is $2m$-times differentiable and that is identically equal to $0$ on $\R \setminus (0,T+1)$ and that is identically equal to $1$ on the interval $[1,T]$. Moreover, for any $2m$-times differentiable function $f$, with derivatives in $L^p(S)$ for $p \in [1,\infty]$ and interval $S \subseteq \R$, we have
$$
\|\Partial[k]{(f b_T)}\|_{L^p(S)} \leq \frac{(8(2m+3/2))^{k+1}}{\pi^{2m+3/2}} \max_{0 \leq h \leq k}\|\Partial[h]{f}\|_{L^p(S \cap [0,T+1])}.
$$
\end{lemma}
\begin{proof}
Consider the function 
$$B(x) = \frac{2 \Gamma(3/2+2m)}{\sqrt{\pi} \Gamma(1+2m)} (1-4x^2)_+^{2m},$$
where $\Gamma$ is the gamma function. Then $B$ is supported on $(-1/2,1/2)$, integrates to $1$, and is $2m$-times differentiable everywhere. Its Fourier transform (see \cite{stein1971introduction}, Thm. 4.15) is 
$$\widehat{B}(\omega) = \Gamma(3/2+2m) J_{\frac{1}{2} + 2m}(\pi |\omega|) \left(\frac{2}{\pi |\omega|}\right)^{2m+1/2} ,$$
where $J_\nu$ is the Bessel J function of order $\nu$ \cite{stein1971introduction}. 
\begin{figure}[t]
  \includegraphics[width=\columnwidth]{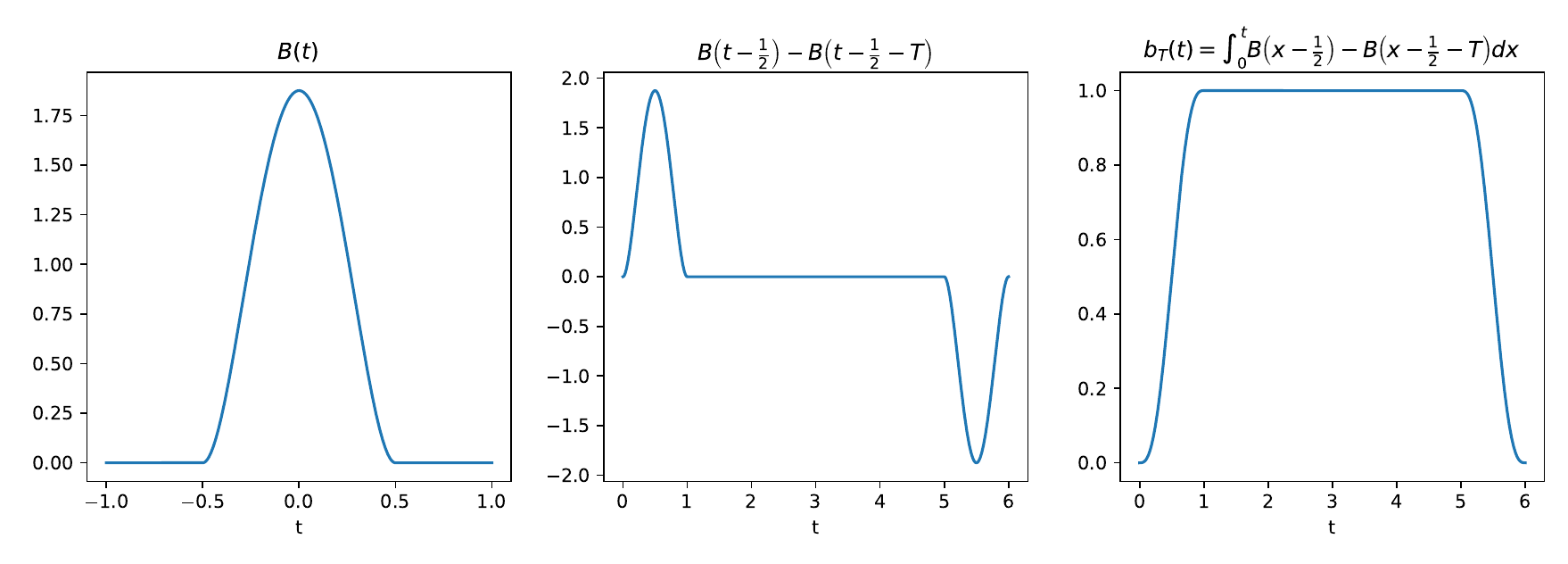}
  \caption{Plots demonstrating the functions used in the construction of
    $b_{T}(t)$ for $m=1$ and $T=5$. The function $B(x)$ is a simple bump function,
    designed such that $\int_{\R}B(x)dx=1$, shown on the left.
    The center demonstrates how we can combine translations of $B(x)$ to get a
    function with two bumps which will eventually cancel out when integrated over $[0,T+1]$, leading to
    the function $b_{T}$ shown on the right.}
  \label{fig:bump}
\end{figure}
Since $J_\nu$ is analytic on $[0, \infty)$ for $\nu > 0$ and $J_\nu(z) = \cO(|z|^\nu)$ when $|z| \to 0$ and also $J_\nu(z) = \cO(z^{-1/2})$ for $z \to \infty$, then $\widehat{B}$ is in $L^1 \cap L^\infty$ and analytic. We build $b_T$ as follows
$$
b_T(t) = \int_0^t B(x-1/2) - B(x - T-1/2) dx.
$$
This function by construction is $2m$-times differentiable everywhere, moreover
it is identically equal to $0$ on $\R\setminus (0,T+1)$ and identically equal to
$1$ on $[0,T]$.
To help with intuitions, we  show an example of the functions $B(x)$, $B(x-\half)-B(x-\half-T)$, and $b_{T}(t)$
for $m=1$ and $T=5$ in \Cref{fig:bump}. 
The Fourier transform of $b_T$ is
$$
\widehat{b}_T (\omega) = \frac{1}{2\pi i \omega} \widehat{B}(\omega) e^{-\pi i \omega }(1 - e^{-2\pi i T \omega}).
$$
Now we define $u(t) := f(t) b_T(t)$. By construction, $u$ is equal to $f_T$ on $[1,T]$ since $b_T$ is identically $1$ on this interval, let $Z := S \cap [0, T+1]$, we have
\begin{align*}
\|\Partial[k]{u}\|_{L^p(S)} &= \|\Partial[k]{u}\|_{L^p(Z)}  = \left\|\sum_{h=0}^k\binom{k}{h}\Partial[h]{f}\Partial[k-h]{b_T}\right\|_{L^p(Z)} \\
& \leq \max_{0 \leq j \leq k}\|\Partial[h]{f}\|_{L^p(Z)}   \sum_{h=0}^l\binom{k}{h} \|\Partial[j]{b_T}\|_{L^\infty(\R)},
\end{align*}
Now, given the Fourier transform of $\widehat{b}_T$,
\begin{align*}
\|\Partial[h]{b_T}\|_{L^\infty} & \leq (2\pi)^h \|\omega^h \widehat{b}_T\|_{L^1} \leq \Gamma(3/2+2m) \frac{2^{2m+h-1/2}}{\pi^{2m-h+3/2}}\|\omega^{-2m-3/2+h} J_{2m + 1/2}(\pi |\omega|)\|_{L^1}.
\end{align*}
Using the fact that $|J_\nu(z)| \leq \min(z^\nu 2^{-\nu}/\Gamma(1+\nu), \nu^{-1/3})$ (see \cite{olver2010nist}, Eq. 10.14.2, 10.14.4) for any $z \geq 0$, for any $\alpha > 0$, we have
\begin{align*}
\int_{\R} \frac{|J_{2m + 1/2}(\pi |\omega|)|}{|\omega|^{-2m-3/2+h}} d \omega & = 2\int_0^\alpha \frac{|J_{2m + 1/2}(\pi x)|}{x^{2m+3/2-h}} d x + 2\int_{\alpha}^\infty \frac{|J_{2m + 1/2}(\pi x)|}{x^{2m+3/2-h}} d x \\
& \leq 2\int_0^\alpha \frac{x^{2m+1/2} 2^{-2m-1/2}}{x^{2m+3/2-h} \Gamma(2m+3/2)} d x + 2\int_{\alpha}^\infty \frac{(2m + 1/2)^{-1/3}}{x^{2m+3/2-h}} d x \\
& = \frac{2^{-2m+1/2} \alpha^h}{h \Gamma(2m+3/2)} + \frac{2\alpha^{-(2m -h +1/2)}}{(2m-h+1/2)(2m + 1/2)^{1/3}}.
\end{align*}
Optimizing in $\alpha$, we obtain
$$
\|\omega^{-2m-3/2+h} J_{2m + 1/2}(\pi |\omega|)\|_{L^1} \leq 2^{-2m+5/2+h}\Gamma(2m+3/2)^{-\frac{2m+1/2 - h}{2m+1/2}},
$$
leading to
$$
\|\Partial[h]{b_T}\|_{L^\infty} \leq \frac{2^{2h+2}}{\pi^{2m-h+3/2}}\Gamma(2m+3/2)^{\frac{h}{2m+1/2}} \leq \frac{2 (8m+6)^{h+1/2}}{\pi^{2m-h+3/2}},
$$
this leads to
\begin{align*}
\|\Partial[k]{u}\|_{L^p(S)}&\leq \max_{0 \leq j \leq k}\|\Partial[h]{f}\|_{L^p(S)}   \sum_{h=0}^l\binom{k}{h} \|\Partial[j]{b_T}\|_{L^\infty(\R)} \\
&\leq \frac{2(8m+6)^{1/2}}{\pi^{2m+3/2}}\left(1+\frac{8m+6}{\pi}\right)^k\max_{0 \leq j \leq k}\|\Partial[h]{f}\|_{L^p(Z)}.
\end{align*}
To conclude, note that
$$
\frac{2(8m+6)^{1/2}}{\pi^{2m+3/2}}\left(1+\frac{8m+6}{\pi}\right)^k \leq \frac{8^{k+1}(2m+3/2)^{k+1}}{\pi^{2m+3/2}}.
$$

\end{proof}

\begin{lemma}\label{lm:bound-L1-norm-with-fourier}
Let $f$ be such that its Fourier transform $\hat{f}$ and its weak derivative $\Partial{\hat{f}}$ are both $L^2$, then
$$\|f\|_{L^1(\R)} \leq \frac{2}{\sqrt{\pi}} \|\hat{f}\|^{1/2}_{L^2(\R)} \|\Partial{\hat{f}}\|^{1/2}_{L^2(\R)}.$$
\end{lemma}
\begin{proof}
For any $t > 0$,
$$\|f\|_{L^1(\R)} \leq \|f\|_{L^1([-t,t])} + \|f\|_{L^1(\R\setminus[-t,t])}.$$
Now by the Hölder inequality we have
$$\|f\|_{L^1([-t,t])} = \|f \cdot 1\|_{L^1([-t,t])} \leq \|f\|_{L^2([-t,t])} \|1\|_{L^2([-t,t])} = \sqrt{2 t} \|f\|_{L^2([-t,t])}$$
where $1(x)$ is the constant function $1$. 
Similarly,
$$\|f\|_{L^1(\R\setminus[-t,t])}  = \|f x \cdot 1/x\|_{L^1(\R\setminus[-t,t])} \leq \|f x\|_{L^2(\R\setminus[-t,t])} \|1/x\|_{L^2(\R\setminus[-t,t])} = \sqrt{2/t} \|f x\|_{L^2(\R\setminus[-t,t])},$$
since $\|1/x\|^2_{L^2(\R\setminus[-t,t])} = 2\int_{t}^\infty 1/x^2 dx = 2/t$. This leads to
$$
\|f\|_{L^1(\R)} \leq \sqrt{2}( \sqrt{t} \|f\|_{L^2} + 1/\sqrt{t} \|f x\|_{L^2
}).$$
Optimizing on $t$ we obtain,
$$ \|f\|_{L^1} \leq 2\sqrt{2} \|f\|^{1/2}_{L^2} \|f x\|^{1/2}_{L^2}.$$
The first case is concluded by applying the Plancherel theorem for which $\|f\|_{L^2} = \|\hat{f}\|_{L^2}$ and $\|f x\|_{L^2} = \|F[f x]\|_{L^2}$ and by the fact that $F[f(x) x](\omega) = i/(2\pi) \Partial{\hat{f}(\omega)}$.
\end{proof}

\label{app:discrete-pl}
\DiscretePL*
\begin{proof}
We will build $u$ as follows
$$
u(t) = f(t) b_T(t),
$$
where $b_T$ is a $m+1$-times differentiable function that is supported on $[0,T+1]$ and that is equal to $1$ on $[1,T]$, while $f$ is a function that interpolates $v_t$, i.e. $f(t) = v_t$ for $t \in \{1,\dots,T\}$. We build $f$ as follows. Consider the following function, that is a product of two sinc functions
$$S(x) := \frac{\sin(\pi x)}{\pi x} \frac{2 \sin(\pi x/2)}{\pi x},$$
$S$ satisfies $S(0) = 1$ and $S(t) = 0$ on $t \in \mathbb{Z} \setminus \{0\}$. Now we can build $f$ as follows
$$f(t) = \sum_{\ell=0} v_{\ell} S(t-\ell).$$
By construction, for all $t \in \{1,\dots, T\}$ the following holds
$$f(t) = \sum_{\ell=0} v_{\ell} S(t-\ell) = \sum_{\ell=0} v_{\ell} \delta_{t=\ell} = v_t.$$
Note that, by  construction $f$ is a band-limited function with band $[-3/4, 3/4]$ that interpolates the given points.

\textbf{Step 1. Bounding $\|\Partial{f}\|_{L^1}$.}
Now we bound pointwise $|\Partial{f(x)}|$ with $|v(x)-v(x-1)|$. The Fourier transform $\hat{f}$ of $u$ is equal to
$$\hat{f}(\omega) = \sum_{\ell=1}^T v_\ell e^{2\pi i \ell \omega} \hat{S}(\omega) = g(\omega) \hat{S}(\omega),$$
where $g(\omega) = \sum_{\ell=1}^T v_\ell e^{2\pi i \ell \omega}$ and where $\hat{S}$ is Fourier Transform of $S$. Note that $S$ is band-limited, i.e., $\hat{S}$ is equal to $0$ on $\R \setminus [-3/4,3/4]$. In particular,
$$\hat{S}(\omega) = |\omega-3/4| - |\omega-1/4| - |\omega+1/4| + |\omega+3/4|, \quad \forall \omega \in \R.$$
Now passing by the Fourier transform of $\Partial{f}$, we obtain 
$$
F[\Partial{f}](\omega) = g(\omega) \hat{S}(\omega) \omega = \hat{L}(\omega) \, \hat{M}(\omega), \quad \hat{L}(\omega) := g(\omega) \frac{1-e^{2\pi i \omega}}{1+\omega^2}, \qquad \hat{M}(\omega) := \hat{S}(\omega) \frac{\omega(1+\omega^2)}{1-e^{2\pi i \omega}}.
$$
Note that $\hat{M}(\omega)$ is bounded, continuous and supported in $[-3/4,3/4]$, since $\hat{S}$ is supported in $[-3/4,3/4]$ and $\frac{\omega (1+\omega^2)}{1-e^{2\pi i \omega}}$ is bounded and analytic on such interval.
So we have
$$
\Partial{f(x)} = F^{-1}[F[\Partial{f}]](x) = F^{-1}\left[\hat{L}(\omega) \hat{M}(\omega)\right] = \left(L \star M \right)(x),
$$
where $\star$ is the convolution operator and $M = F^{-1}[\hat M]$, $L = F^{-1}[\hat L]$. Now note that
\begin{align*}
\hat{L}(\omega) := g(\omega) \frac{1-e^{2 \pi i \omega}}{1+\omega^2} &=
\sum_{\ell=1}^T v_\ell \frac{e^{2 \pi i \ell \omega }}{1+\omega^2} - \sum_{\ell=1}^T v_\ell \frac{e^{2 \pi i (\ell+1)\omega }}{1+\omega^2} \\
& = \sum_{\ell=2}^T (v_\ell - v_{\ell - 1}) \frac{e^{2 \pi i \ell \omega }}{1+\omega^2} ~~+~~ v_1 \frac{e^{2 \pi i \omega }}{1+\omega^2}
\end{align*}
Since the inverse Fourier transform of $e^{2\pi a \omega}/(1+\omega^2)$ is $e^{-|x-a|}$ for any $a \in \R$, we have
$$
L(x) := F^{-1}\left[g(\omega) \frac{1-e^{2 \pi i \omega}}{1+\omega^2}\right](x) = \pi \sum_{\ell=2}^T (v_\ell - v_{\ell-1}) e^{-2\pi|x - \ell|} ~~+~~ v_1 \pi e^{-2\pi|x - 1|}.
$$
By Young's inequality for the convolution, 
we have that 
$\|f\star g\|_{L^1} \leq \|f\|_{L^1} \|g\|_{L^1}$ for any integrable functions $f,g$, so in our case
\begin{align*}
\|\Partial{f}\|_{L^1} 
& = \| L \star M \|_{L^1} \leq \|L\|_{L^1} \| M \|_{L^1} \\
& \leq \pi \| M \|_{L^1} \sum_{\ell=2}^T  \|v_{\ell} - v_{\ell-1}\| \|e^{-2\pi|x-\ell|}\|_{L^1} ~~+~~ \pi \| M \|_{L^1} \|v_1\| \|e^{-2\pi|x-1|}\|_{L^1}.
\end{align*}
To conclude we have $\|e^{-2\pi|x-\ell|}\|_{L^1} = \|e^{-2\pi|x|}\|_{L^1}  = 1/\pi$ and we need to bound the $L^1$ norm of $M$. Note that $\hat{M}(\omega)$ admits a weak derivative, since it is the product of a bounded analytic function on the support and $\hat{S}$ that admits a weak derivative that is the following
$$\Partial{\hat{S}(\omega)} = \operatorname{sign}(\omega-3/4) - \operatorname{sign}(\omega-1/4) - \operatorname{sign}(\omega+1/4) + \operatorname{sign}(\omega+3/4).$$
In particular, we have for every $\omega \in \R$
$$
|\hat{M}(\omega)| = \left|\hat{S}(\omega) \frac{\omega(1+\omega^2)}{1-e^{2\pi i \omega}}\right| \leq  1_{[-3/4,3/4]}(\omega)
$$
$$
|\Partial{\hat{M}(\omega)}| = \left|\Partial{\hat{S}(\omega)} \frac{\omega(1+\omega^2)}{1-e^{2\pi i \omega}} + \hat{S}(\omega) \frac{\partial}{\partial \omega}\frac{\omega(1+\omega^2)}{1-e^{2\pi i \omega}}\right| \leq 2 \times 1_{[-3/4, 3/4]} (\omega)
$$
Using \cref{lm:bound-L1-norm-with-fourier} we can bound $\|M\|_{L^1}$ as follows,
$$
\|M\|_{L^1} \leq \frac{2}{\sqrt{\pi}} \|\hat{M}\|^{1/2}_{L^2}\|\partial \hat{M}\|^{1/2}_{L^2} \leq 2,
$$
obtaining
\begin{equation}\label{eq:bound-f-L1}
\|\Partial f\|_{L^1(\R)} \leq 2\|v_1\|_\ww+ 2\sum_{t=2}^T \|v_t - v_{t-1}\|_\ww.
\end{equation}

\textbf{Step 2. Bounding $\|\Partial[k]f\|_{L^\infty}$.}
Let $k \in \N$ and $k \geq 1$. Since $\widehat{Z}(\omega) := \widehat{S}(\omega/4)$ is equal to $1$ on $[-1, 1]$ and it is supported on $[-3,3]$, we have that $\widehat{f}(\omega) = \widehat{Z}(\omega) \widehat{f}(\omega)$. So, by using the properties of the Fourier transform of convolutions, we have that for all $t \in \R$,
$$
f(t) = F^{-1}[\widehat{f}](t) = F^{-1}[\widehat{Z} \widehat{f}](t) = (F^{-1}[\widehat{Z}] \star F^{-1}[f]) = (Z \star f)(t),
$$
where $Z(t) = 4 S(4 t)$ for all $t \in \R$ is the inverse Fourier transform of $\widehat{Z}$.
So we have
$$
\|\Partial[k]{f}\|_{L^\infty} = \|\Partial[k]{(Z \star f)}\|_{L^\infty} = \|(\Partial[k]{Z})\star f\|_{L^\infty} \leq \|\Partial[k]{Z}\|_{L^1} \|f\|_{L^\infty}.
$$
Now, we have
\begin{align}\label{Linfty-interpolator}
\|f\|_{L^\infty} = \sup_{t \in \R} \left|\sum_{\ell=1}^T v_\ell S(t-\ell)\right| \leq (\max_{\ell \in [T]} \|v_\ell\|_{\cal W}) \sup_{t \in \R} \sum_{\ell=1}^T |S(t-\ell)|.
\end{align}
Since $|S(t)| \leq 1/(1+ 2t^2)$ for any $t \in \R$, we have
$$
\sup_{t \in \R} \sum_{\ell=1}^T |S(t-\ell)| \leq \sup_{t \in \R} \sum_{\ell \in \mathbb{Z}} \frac{1}{1+ 2(t-\ell)^2} = \sum_{\ell \in \mathbb{Z}} \frac{1}{1+ 2\ell^2} \leq 4.
$$
To conclude this section, using \cref{lm:bound-L1-norm-with-fourier} and since $F[\nabla^k Z] = \omega^k Z/(2\pi i)^k$
$$\|\Partial[k]{Z}\|_{L^1} \leq \frac{2\sqrt{2}}{(2\pi)^k}\|\omega^k\widehat{Z}\|^{1/2}_{L^2}\|\Partial{(\omega^k\widehat{Z})}\|^{1/2}_{L^2} \leq \frac{3^{2+k}+4k}{(2\pi)^k(1+k^{5/2})} \leq 3.$$
So, for any $k \in \N$ (including $0$, since we have \cref{Linfty-interpolator})
\begin{equation}\label{eq:bound-f-Linfty}
\|\Partial[k]{f}\|_{L^\infty} \leq \|\Partial[k]{Z}\|_{L^1} \|f\|_{L^\infty} \leq 12 \max_{\ell} \|v_\ell\|_{\cal W}
\end{equation}

\textbf{Step 3. Building $b_T$ and computing the final norms.}
\cref{lm:construction-of-bT} constructs a function $b_T$ that is $2m$-times differentiable and that is identically equal to $0$ on $\R\setminus (0,T+1)$ and that is identically equal to $1$ on $[1,T]$, moreover it proves that for any $2m$-times differentiable function that has the derivatives $L^p(S)$ integrable for some $p \in [0,\infty]$ and some interval $S \subseteq \R$, we have
$$
\|\Partial[k]{(f b_T)}\|_{L^p(S)} \leq C_{k,m} \max_{0 \leq h \leq k} \|\Partial[h]{f}\|_{L^p(S \cap [0,T+1])}.
$$
Define $u(t) := f(t) b_T(t)$. By construction $u$ is equal to $f_T$ on $[0,T]$ since $b_T$ is identically $1$ on this interval, and so in particular
$$
u(\ell) = f(\ell) b_T(\ell) = f(\ell) = v_\ell, \quad \forall \ell \in \{1,\dots, T\},
$$
moreover $u(t) = f(t) b_T(t) = 0$ for $t \in \R\setminus(0,T+1)$. By applying the result above, together with \cref{eq:bound-f-Linfty}
$$
\|u - \Partial[2m]{u}\|_{L^\infty} \leq \|u\|_{L^\infty} + \|\Partial[2m]{u}\|_{L^\infty} \leq 12(1+C_{2m,m}) \max_{\ell} \|v_\ell\|_{\cal W}.
$$
Applying the same lemma, with \cref{eq:bound-f-L1} we have
$$
\|\Partial{u}\|_{L^1(\R)} \leq 2C_{1,m} \|v_1\|_\ww+ 2C_{1,m} \textstyle\sum_{t=2}^T \|v_t - v_{t-1}\|_\ww. 
$$
This completes the proof.
\end{proof}

\subsection{Proof of Lemma~\ref{lemma:bound-R}}

Here, we recall a classical result about the positivity of the sine transform of a positive decreasing function (see e.g. \cite{tuck2006positivity} Eq. 4).
\begin{lemma}\label{lemma:bound-R}
Let $R$ be an integrable, positive and strictly decreasing on $(0, \infty)$. Then, for any $t > 0$ we have
$$
\int_0^\infty R(x) \sin(2\pi x t)  dx \geq 0.
$$
\end{lemma}
\begin{proof}
Since $\sin(2\pi (z+1/2)) = -\sin(2\pi z)$ for each $z \in [j,j+1/2]$ and $j \in \N$, we have
\begin{align*}
    \int_0^\infty R(x) \sin(2\pi x t) dx & = \sum_{j=0}^\infty \int_\frac{j}{t}^{\frac{j+1}{t}} R(x) \sin(2\pi x t) dx \\
    & = \frac{1}{t}\sum_{j=0}^\infty \int_0^{1} R\left(\frac{j+\theta}{t}\right) \sin(2\pi \theta) d\theta \\
    & = \frac{1}{t}\sum_{j=0}^\infty \int_0^{1/2} \left[R\big(\frac{j+\theta}{t}\big) - R\big(\frac{j+1/2+\theta}{t}\big)\right] \sin(2\pi \theta) d\theta \\
    & \geq 0,
\end{align*}
 where the last step is due to the fact that $R$ is decreasing, so $R((j+\theta)/t) - R((j+\theta)/t + 1/2t) \geq 0$ for any $j \in \N, \theta \in [0,1/2]$ and that $\sin(2\pi \theta)  \geq 0$ on the integration interval $\theta \in [0,1/2]$.
\end{proof}

\subsection{Proof of Proposition~\ref{prop:horizon-free-density}}%
\label{app:horizon-free-density}
\HorizonFreeDensity*
\begin{proof}
\textbf{Step 1. Characterization of $Q$.}
The function $Q$ is even, strictly positive and analytic on $\R \setminus \{0\}$. To study its integrability define the auxiliary function $S: [0,\infty) \to [0, \infty)$ as
$$S(z) = \frac{\log \log \pi}{2 \log \log(\pi + 1/z^s)}.$$
$S$ is concave, strictly increasing, on $(0,\infty)$ and with $S(0)=0$ and $\lim_{z\to \infty} S(z) = 1/2$. So its derivative $S'$ corresponding to 
$$S'(z) = \frac{s/2 \, \log\log \pi}{z~\,(1+ \pi z^s) \,\, \log(\pi + z^{-s}) \,\,\log^2\log(\pi+z^{-s})},$$
is positive and strictly decreasing on $(0, \infty)$ and since $0< s < 2m$, the function $(\cdot)^{s/2m}$ is concave, and we have $(1+(|\omega|/2\pi)^{2m})^{s/2m} \geq 2^{\frac{s}{2m}-1} (1+ (|\omega|/2\pi)^s)$, so
\begin{align}\label{eq:bound-Q-over-Sprime}
    L := \sup_{\omega \in \R} \frac{Q(\omega)}{S'(|\omega|)} = \sup_{\omega \in \R} \frac{1+ \pi |\omega|^s}{(1+(|\omega|/2\pi)^{2m})^{s/2m}} \leq 2^{1-\frac{s}{2m}}\sup_{\omega \in \R} \frac{1+\pi|\omega|^s}{1 + (|\omega|/2\pi)^s} = 2^{\frac{2m(1+s) - s}{2m}} \pi^{1+s}.
\end{align}
So $0 < Q(\omega) \leq L S'(|\omega|)$ for any $\omega \in \R$, and since $S'$ is positive and integrable, then $Q$ is integrable too and we have
$$
\int_{\R} Q(\omega)d \omega \leq 2 L \int_0^\infty S'(z)dz = 2 L (\lim_{z \to \infty} S(z)  - S(0))= L.
$$
To conclude this first step, since $Q$ is integrable it admits a Fourier transform $\widehat{Q}$, and since it is also positive, $k(t,t') := \widehat{Q}(t-t')$ for $t,t' \in \R$ is a translation invariant kernel. In particular, $k(t,t) = \widehat{Q}(0) = \int_\R Q(\omega) d\omega \leq L$, for any $t \in \R$.

\textbf{Step 2. Characterization of $R$ and explicit bound for $\|F[R]\|_{L^1[-T,T]}$.}
The second condition to apply \cref{thm:rkhs-to-pl}, concerns the function $R$ defined as
$$
R(x) := \frac{2 \pi}{x (1 + (x/2 \pi)^{2m}) Q(x)}.
$$
Note that $R$ is an odd function, since $x$ is odd, while both $1 + (\frac{x}{2 \pi})^{2m}$ and $Q(x)$ are even. Moreover, $R$ is analytic on $\R \setminus \{0\}$ since $Q(x)$ is analytic on the same domain and $x(1+(x/2 \pi)^{2m})$ is analytic on the whole axis. Expanding the definition of $Q$ in $R$, we obtain
$$
R(\omega) = \frac{ C_0 \log(\pi + |\omega|^{-s}) \log^2(\log(\pi + |\omega|^{-s}))}{(1 + (\omega/2 \pi)^{2m})^{\frac{2m-s}{2m}}},
$$
where $C_0 = 4\pi/(s \log \log \pi)$. From which we observe that $R$ is also positive and strictly decreasing on $(0, \infty)$ since,  $ \log(\pi + |x|^{-s})$, $\log(\log(\pi + |x|^{-s}))^2$ and $1/(1 + (\omega/2 \pi)^{2m})^{\frac{2m-s}{2m}})$ are strictly positive and strictly decreasing on $(0, \infty)$. So we can apply the bound on $\|F[R]\|_{L^1([-T,T])}$ in \cref{thm:rkhs-to-pl}, obtaining, for any $\alpha > 0$
\begin{align*}
c(T) := \|F[R]\|_{L^1([-T,T])} \leq 2\pi T^2 \int_0^\alpha  R(x) x dx + \frac{2}{\pi} \int_\alpha^\infty \frac{R(x)}{x} dx. 
\end{align*}
To bound such integrals, we first simplify $R$.
Let $\beta, \gamma > 0$, since the following functions are bounded, non-negative, with a unique critical point that is a maximum, by equating their derivative to zero we obtain
$$
\sup_{z > \pi} \frac{\log^2(\log(z))}{\log^{1+\gamma}(z)} = (2/\gamma)^2 e^{-2}, \qquad \sup_{z > \pi} \frac{\log^{1+\gamma}(z)}{z^\beta} = (\frac{1+\gamma}{\beta})^{1+\gamma} e^{-1-\gamma},
$$
so we have for any $x > 0$,
\begin{align*}
R(x) &= \frac{ C_0 \log(\pi + x^{-s}) \log(\log(\pi + x^{-s}))^2}{(1 + (x/2 \pi)^{2m})^{\frac{2m-s}{2m}}}\\
& = \frac{\log^2(\log(\pi + x^{-s}))}{\log^\gamma(\pi + x^{-s})}\frac{\log^{1+\gamma}(\pi + x^{-s})}{(\pi + x^{-s})^\beta}\frac{C_0 (\pi + x^{-s})^\beta}{(1 + (x/2 \pi)^{2m})^{\frac{2m-s}{2m}}} \\
&\leq C_1(\beta, \gamma) \frac{(\pi + x^{-s})^\beta}{(1 + (x/2 \pi)^{2m})^{\frac{2m-s}{2m}}}.
\end{align*}
where
$C_1(\beta,\gamma) =  (2/\gamma)^2 (\frac{1+\gamma}{\beta})^{1+\gamma} e^{-3 -\gamma} C_0 \leq 16 C_0/(e^3\gamma^2\beta^{1+\gamma})  \leq \gamma^{-2} \beta^{-1-\gamma} C_0 $.
Now we can control the integral of interest by using the bound above. First, we will split it in two regions of interest. 
For the first term, letting $\beta < 1$,
\begin{align*}
 T^2 \pi \int_0^\alpha R(x) x dx &\leq T^2 \pi C_1(\beta,\gamma) \int_0^\alpha \frac{\pi^\beta + x^{-s\beta}}{(1 + (x/2 \pi)^{2m})^{\frac{2m-s}{2m}}} x dx \\
 &\leq T^2 \pi C_1(\beta,\gamma) \int_0^\alpha (\pi^\beta + x^{-s\beta}) x dx \\
 & = C_1(\beta,\gamma) ~T^2 \alpha^2 ~\left(\frac{\pi^{1+\beta}}{2} + \frac{\pi \alpha^{-\beta s}}{2-\beta s}\right)
\end{align*}
For the second term, we have
\begin{align*}
\frac{1}{\pi} \int_\alpha^\infty \frac{R(x)}{x} dx &\leq \frac{C_1(\beta, \gamma)}{\pi} \int_\alpha^\infty \frac{(\pi+x^{-s})^{\beta}}{x (1 + (x/2 \pi)^{2m})^{\frac{2m-s}{2m}}}dx \\
& \leq \frac{C_1(\beta, \gamma)}{\pi} \int_\alpha^\infty \frac{\pi^\beta + \alpha^{-\beta s}}{x (1 + x^{2m-s}/(2\pi)^{2m-s})} dx \\
& \leq \frac{C_1(\beta, \gamma)}{\pi} \int_\alpha^{2\pi} \frac{\pi^\beta + \alpha^{-\beta s}}{x} dx + \frac{C_1(\beta, \gamma)}{\pi} \int_{2\pi}^\infty \frac{\pi^\beta + \alpha^{-\beta s}}{x^{2m+1-s}/(2\pi)^{2m-s}} dx \\
& = \frac{C_1(\beta, \gamma)}{\pi} (\pi^\beta + \alpha^{-s\beta}) \log(\frac{2\pi}{\alpha}) + \frac{C_1(\beta, \gamma)}{\pi} (\pi^\beta + \alpha^{-\beta s})/(2m-s).
\end{align*}
So $c(T)$ is bounded by
$$
c(T) \leq \frac{2 \pi T^2 \alpha^2}{\gamma^2 \beta^{1+\gamma}} ~\left(\frac{\pi^{\beta}}{2} + \frac{ \alpha^{-\beta s}}{2-\beta s}\right) + \frac{2(\pi^\beta + \alpha^{-s\beta}) \log(\frac{2\pi}{\alpha})}{\pi \gamma^2 \beta^{1+\gamma}}  + \frac{2\pi^\beta + 2\alpha^{-\beta s}}{\pi \gamma^2 \beta^{1+\gamma} (2m-s)}.
$$
By choosing $\alpha = 1/T$, $\beta = 1/\log(T)$, $\gamma = 1/\log(\log(T))$, we have
$$T \alpha = 1, \quad \alpha^{-\beta} = T^{1/\log(T)} = e, \quad \beta^{1+\gamma} = (\log(T))^{1/\log(\log(T))} = e,$$ 
and, since $s \geq 1$ (by assumption), $2m-s \geq 1$ (by definition of $m$) and $T > 3$ (by assumption), we have
\begin{align*}
c(T) &\leq C_0 \log(T) \log^2 \log(T)  \left(\frac{2\pi^2e}{2} + \frac{2\pi e^{1+s}}{2-s} + \frac{2(\pi + e^s) \log(2\pi T) }{\pi}  + \frac{2\pi + e^s}{\pi (2m-s)} \right)\\
& \leq  C_0 \log^2(T) \; \log^2 \log(T) \; \left(\frac{2\pi^2e}{2} + 2\pi e^2 + \frac{2(\pi + 2e) \log(2\pi)}{\pi}  + \frac{2e}{\pi} + 2\right) \\
& \leq (4 \pi^2 e^2)^2 \log^2(T) \; \log^2 \log(T).\qedhere
\end{align*}

\end{proof}


\section{Proofs for Section~\ref{sec:curvature} (\SecCurvature)}%
\label{app:curvature}
\LossCurvature*
\begin{proof}
  Let $x=X\phi(t)$ and $y=Y\phi(t)$. By definition and $\beta$-exp-concavity of $\ell_{t}$, we have
  \begin{align*}
    \elltilde_{t}(X)-\elltilde_{t}(Y)
    &=
      \ell_{t}(x)-\ell_{t}(y)
      \le
      \inner{\grad\ell_{t}(x), x-y}-\frac{\beta}{2}\inner{\grad\ell_{t}(x), x-y}^{2}_{\ww}\\
    &=
      \inner{\grad\ell_{t}(x), (X-Y)\phi(t)}_{\ww}-\frac{\beta}{2}\inner{\grad\ell_{t}(x), (X-Y)\phi(t)}^{2}_{\ww}\\
    &=
      \inner{\grad\ell_{t}(x)\otimes \phi(t), (X-Y)}_{\knorm}-\frac{\beta}{2}\inner{\grad\ell_{t}(x)\otimes \phi(t), (X-Y)}^{2}_{\knorm}.
  \end{align*}
  Observing that
  $\grad\elltilde(X)=\grad\ell_{t}(x)\otimes \phi(t)\in\WW$
  completes the proof.
\end{proof}

\subsection{Strongly-convex Losses}\label{app:sc}

In this section we show how to apply our static-to-dynamic reduction
in the context of strongly-convex losses. Interestingly, the
algorithm ends up being essentially the same as
the Kernelized-ONS algorithm of \cite{jezequel2019efficient}, but with a weighted norm defined in terms of the feature covariance operator,
$\Sigma_{t}=\lambda I+\beta\sum_{s=1}^{t}\phi(s)\otimes\phi(s)$.
The following
lemma shows how to connect the instantaneous regret on
round $t$ to the kernelized linear losses $\gt\otimes\phi(t)$ and
is analogous to \Cref{prop:loss-curvature}. %
\begin{restatable}{proposition}{SCLossCurvature}\label{prop:sc-loss-curvature}
  Let $\ell_{t}:\ww\to\R$ be a $\beta$-strongly-convex function, let $\hh$ be an
  RKHS with associated feature map $\phi(t)\in\hh$, and define
  $\elltilde_{t}(\W\phi(t))$ for $\W\in \Lin(\hh,\ww)$. Then for any $X,Y\in \Lin(\hh,\ww)$,
  \begin{align*}
    \elltilde_{t}(X)-\elltilde_{t}(Y)\le \inner{\elltilde_{t}(X),X-Y}_{\knorm}-\frac{\beta}{2}\inner{(X-Y)(\phi(t)\otimes\phi(t)), X-Y}_{\knorm},
  \end{align*}
  where $\phi(t)\otimes\phi(t):\hh\to\hh$ is the operator with action
  $(\phi(t)\otimes \phi(t))h = \inner{\phi(t),h}_{\hh}\phi(t)$.
\end{restatable}
\begin{proof}
  Let $x=X\phi(t)$ and $y=Y\phi(t)$, and observe that by
  $\beta$-strong-convexity of $\ell_{t}$ in $\ww$ we have
  \begin{align*}
    \elltilde_{t}(X)-\elltilde_{t}(Y)&=\ell_{t}(x)-\ell_{t}(y)
    \le
      \inner{\grad\ell_{t}(x),x-y}_{\ww} - \frac{\beta}{2}\norm{x-y}^{2}_{\ww}\\
    &=
      \inner{\grad\ell_{t}(x), (X-Y)\phi(t)}_{\ww}-\frac{\beta}{2}\inner{(X-Y)\phi(t),(X-Y)\phi(t)}_{\ww}\\
    &\overset{(\star)}{=}
      \inner{\grad\ell_{t}(x)\otimes\phi(t), X-Y}_{\knorm}-\frac{\beta}{2}\inner{(X-Y)\phi(t)\otimes\phi(t),(X-Y)}_{\ww}\\
    &=
      \inner{\grad\elltilde_{t}(x), X-Y}_{\knorm}-\frac{\beta}{2}\inner{(X-Y)\phi(t)\otimes\phi(t),(X-Y)}_{\ww}~,
  \end{align*}
  where $(\star)$ uses \Cref{lemma:feature-covariance-norm} and the last line observes that
  $\grad\ell_{t}(x)\otimes\phi(t)=\grad\elltilde_{t}(X)$.\qedhere

\end{proof}
Using this result it is straight-forward to see that the usual ONS arguments
work in this setting. For instance, by running mirror descent with regularizer
$\psi_{t}(\W)=\half \inner{\W\Sigma_{t},W}_{\knorm}$
where $\Sigma_{t}=\lambda I+\beta\sum_{s=1}^{t}\phi(t)\otimes\phi(t)$,
we have the following regret guarantee.

\begin{restatable}{theorem}{SCRegret}\label{thm:sc-regret}
  (K-ONS for Strongly-convex Losses)
  Let $\ell_{1},\ldots,\ell_{T}$ be a sequence of $\beta$-strongly convex losses.
  Let $\lambda>0$ and for all $t$, define
  $\Sigma_{t}=\lambda I+\beta\sum_{s=1}^{t}\phi(s)\otimes\phi(s)$
  and $\norm{\W}_{\Sigma_{t}}^{2}=\inner{\W\Sigma_{t},\W}_{\knorm}$ for $\W\in\Lin(\hh,\ww)$.
  Suppose that on each round $\cA$ updates
  \begin{align*}
    \Wtpp = \argmin_{\W\in\Lin(\hh,\ww)}\inner{\Gt,\W}_{\knorm}+\half\norm{\W-\Wt}^{2}_{\Sigma_{t}},
  \end{align*}
  starting from $\W_{1}=\zeros\in \Lin(\hh,\ww)$.
  Then for any $\cmp_{1},\ldots,\cmp_{T}$ in $\ww$ and $\Cmp\in\Lin(\hh,\ww)$
  satisfying $\cmp(t)=\Cmp\phi(t)$ for all $t$, \Cref{alg:kernel-oco}
  applied with $\cA$ guarantees
  \begin{align*}
    R_{T}(\cmp_{1},\ldots,\cmp_{T})
    &\le
      \frac{\lambda}{2}\norm{\Cmp}^{2}_{\knorm}+\frac{G^{2}}{2\beta}\deff\brac{\lambda/\beta}\Log{e+\frac{e\beta\lambda_{\max}(K_{T})}{\lambda}}~,
  \end{align*}
  where $K_{T}=(\inner{\phi(s),\phi(t)}_{\hh})_{s,t\in[T]}$ and
  $\deff(\lambda)=\Tr{K_{T}(\lambda I + K_{T})^{\inv}}$.
\end{restatable}
\begin{proof}
  Applying \Cref{thm:RKHS-SR-solves-DR} followed by
  \Cref{prop:sc-loss-curvature}, we have
\begin{align*}
  R_{T}(\cmp_{1},\ldots,\cmp_{T})&=\tilde R_{T}(\Cmp)=
  \sumtT \elltilde_{t}(\Wt)-\elltilde_{t}(\Cmp)\\
  &\le
     \sumtT \inner{\Gt, \Wt-\Cmp}_{\knorm} - \frac{\beta}{2}\inner{(\Wt-\Cmp)(\phi(t)\otimes\phi(t)),\Wt-\Cmp}_{\knorm}\\
   &\overset{(a)}{\le}
     \sumtT \half \norm{\Cmp-\Wt}^{2}_{\Sigma_{t}}-\half \norm{\Cmp-\Wtpp}^{2}_{\Sigma_{t}}- \frac{\beta}{2}\inner{(\Wt-\Cmp)(\phi(t)\otimes\phi(t)),\Wt-\Cmp}_{\knorm} \\
   &\qquad
     + \sumtT \inner{\Gt,\Wt-\Wtpp}_{\knorm}-\half\norm{\Wtpp-\Wt}^{2}_{\Sigma_{t}}\\
   &\overset{(b)}{\le}
     \sumtT\half\norm{\Cmp-\Wt}^{2}_{\Sigma_{\tmm}} -\half\norm{\Cmp-\Wtpp}^{2}_{\Sigma_{t}} + \sumtT \frac{1}{2}\norm{\gt}^{2}_{\ww}\norm{\phi(t)}^{2}_{\Sigma_{t}^{\inv}}\\
   &\le
     \frac{\lambda}{2}\norm{\Cmp}^{2}_{\knorm} + \sumtT \frac{G^{2}}{2}\norm{\phi(t)}^{2}_{\Sigma_{t}^{\inv}}\\
   &\overset{(c)}{\le}
    \frac{\lambda}{2}\norm{\Cmp}^{2}_{\knorm} + \frac{G^{2}}{2\beta}\deff\brac{\lambda/\beta}\Log{e+\frac{e\beta\lambda_{\max}(K_{T})}{\lambda}},
\end{align*}
where $(a)$ applies the standard
bound for online mirror descent, $(b)$ observes that
\begin{align*}
  \half\norm{X-Y}^{2}_{\Sigma_{t}}-\frac{\beta}{2}\inner{(X-Y)(\phi(t)\otimes\phi(t)), (X-Y)}_{\knorm}
  &=\half\inner{(X-Y)\Sigma_{\tmm},X-Y}_{\knorm}\\
  &=\half\norm{X-Y}^{2}_{\Sigma_{\tmm}}
\end{align*}
and uses Fenchel-Young inequality to bound
\begin{align*}
\inner{\Gt, \Wt-\Wtpp}_{\knorm}-\half\norm{\Wtpp-\Wt}^{2}_{\Sigma_{t}}\le \half\norm{\Gt\Sigma_{t}^{-\half}}^{2}_{\knorm}=\half\norm{\gt}^{2}_{\ww}\norm{\phi(t)}^{2}_{\Sigma_{t}^{\inv}}
\end{align*}
and $(c)$ uses a mild generalization of the usual log-determinant lemma
(\Cref{lemma:log-det,lemma:log-det-sum})
and defines defines $K_{T}=(\inner{\phi(s),\phi(t)}_{\hh})_{s,t\in[T]}$.
\end{proof}

  Note that in the static regret setting, it is possible to
  avoid the dependence on the comparator norm entirely and pay only the
  logarithmic penalty---we do not expect
  such an improvement to be possible here since it would violate known
  $\Omega(P_{T})$ lower bounds for strongly-convex losses \citep{yang2016tracking}.

\subsection{Additional Details for Example~\ref{ex:spline-kernel}}%
\label{app:sobolev}
In this section we provide some extra details showing that for the RKHS $\hh$
associated with kernel
$k(t,s)= \min(s,t)$, we can bound
$\norm{\cmp}_{\hh}^{2}=\norm{\grad\cmp}_{L^2}^{2}=\cO\Big(\sqrt{\sum_t \norm{\cmp_t-\cmp_\tmm}^2_{\ww}}\Big)$ (\Cref{thm:discrete-squared-pl})
and $\deff(\lambda)=\cO(T/\sqrt{\lambda})$ (\Cref{thm:spline-deff}).
We begin with the bound on the continuous squared path-lenth
$\norm{\grad\cmp}^2_{L^2}$.

\begin{restatable}{theorem}{DiscreteSquaredPL}\label{thm:discrete-squared-pl}
  Let $\hh$ be the RKHS associated to the kernel $k(s,t)=\min(s,t)$ on $[0,T]$.
  Then for any $v_1,\dots, v_T \in \R^d$ there exists a function $u\in\hh$
  such that $\cmp(t)=v_t$ for all $t\in[T]$ and
\begin{align*}
\norm{\cmp}_{\hh}^{2}=\|\Partial{u}\|_{L^2}^2 \leq C \sum_{t=2}^T \|v_t - v_{t-1}\|_\ww^2 + C\|v_1\|_\ww^2,
\end{align*}
with $C \leq \frac{5}{4}$.
\end{restatable}
\begin{proof}

    We assume without loss of generality that $v_t\in\R$ since the result extends immediately to $\R^d$ via the coordinate-wise extension mentioned in \Cref{sec:bg}.
    For brevity we will define $v_t=0$ for $t\notin \Set{1,\ldots,T}$ so that 
    we can write $\sum_{t=1}^T \abs{v_t -v_\tmm}^2 + \abs{v_1}^2=\sum_{t}\abs{v_t-v_\tmm}^2$. 

    Note that the RKHS associated with kernel $k(s,t)=\min(s,t)$ is
    $\hh=\Set{f\in L^{2}:f'\in L^{2}, f(0)=0}$,
    with associated norm
    $\norm{f}_{\hh}=\norm{\grad f}_{L^{2}}=\int\abs{\grad f(x)}^{2}dx$ (see, e.g., Example 23 of \citet{berlinet2011reproducing} with $m=1$).
    Now suppose we define
    \begin{align*}
    u(t) = \sum_{i=1}^T v_i \sinc(t-i)
    \end{align*}
    where $\sinc(x) = \sin(\pi x)/\pi x$.
    Then $u$ and $u'$ are square
    integrable and
    $u(0)=0$, so $u\in\hh$. Moreover, 
    the norm associated with $\hh$ is
    $\norm{f}_{\hh}^{2}=\int \abs{\grad f(x)}^{2}dx=\norm{\grad f}_{L^{2}}^{2}$,
    so we need only show that the constructed function $\cmp(t)$
    has $\norm{\grad u}_{L^{2}}^{2}\le \cO(\sum_{t}\abs{v_{t}-v_{\tmm}}^{2})$.

    Denote $v(t)=\sum_{i}v_{i}\delta(t-i)=v_{t}$ and
    observe that we can write $u(t)=\sum_{i=1}^{T}v(i)\sinc(t-i)=(v\star\sinc)(t)$,
    so using the fact that the Fourier transform of $\sinc$ is the rectangle
    function $1_{[-\half,\half]}(\omega)=\mathbb{I}\Set{\omega\in[-\half,\half]}$
    (see, e.g., \citet{kammler2007first}),
    we have
    $\hat u(\omega)=\hat{v\star\sinc}(\omega)=\hat v(\omega) 1_{[-\half,\half]}(\omega)$.
    Thus,
    \begin{align*}
    \norm{\grad\cmp}_{L^2}^2
    &= \int_\R \abs{\grad\cmp(t)}^2dt
      =
      \int_{\R}\omega^{2}\abs{\hat u(\omega)}^{2}d\omega
    =
    \int_{-\half}^\half \omega^2 \abs{\hat v(\omega)}^2d\omega
    \end{align*}
    via Parseval's identity. We proceed by relating
    $\hat v(\omega)$ to the DFT of the difference sequence,
    $\Delta v(t) = v_t-v_\tmm$ and
    then applying Parseval's inequality for sequences to get 
    $\int \abs{\hat{\Delta v}(\omega)}^2\le \sum_t\abs{\Delta v(t)}^2 = \sum_{t}\abs{v_t-v_\tmm}^2$.

    Observe that
    the DFT of the difference sequence is
    \begin{align*}
    \hat{\Delta v}(\omega) = \sum_{t}(v_t-v_\tmm)e^{-2\pi i \omega t} = (1-e^{-2\pi i\omega})\sum_t v_t e^{-\pi i \omega t}
    = (1-e^{-2\pi i\omega})\hat v(\omega).
    \end{align*}
    Thus,
    \begin{align*}
    \norm{\grad\cmp}^2_{L^2} 
    &=
    \int_{-\half}^\half \omega^2 \abs{\hat v(\omega)}^2d\omega\\
    &=
    \int_{-\half}^\half \frac{\omega^2}{\abs{1-e^{-\pi i \omega}}^2}\abs{\hat{\Delta v}(\omega)}^2 d\omega
    \end{align*}
    Now observe that using the identity $1-\cos(x)=2\sin^2(x/2)$ we have
    \begin{align*}
      \abs{1-e^{2\pi i\omega}}^2 &= (1-e^{-\pi i\omega})(1-e^{\pi i \omega})
                                 = 2 -e^{-\pi i \omega}-e^{\pi i \omega}\\
                               &= 2(1-\cos(\omega))=4\sin^2(\omega/2),
    \end{align*}
    then, using $\omega^2/\sin^2(\omega/2)\le 5$ on $[-\half,\half]$ we have
    \begin{align*}
    \norm{\grad\cmp}_{L^2}^2
    &=
    \int_{-\half}^\half \frac{\omega^2}{4\sin^2(\omega /2)}\abs{\hat{\Delta v}(\omega)}^2 d\omega\\
    &\le
    \frac{5}{4}\int_{-\half}^\half \abs{\hat{\Delta v}(\omega)}^2 d\omega
    =
    \frac{5}{4}\sum_t\abs{\Delta v(t)}^2\\
    &=
    \frac{5}{4}\sum_{t}\abs{v_t-v_\tmm}^2
    \end{align*}
    where the last line applies Parseval's identity for sequences.

\end{proof}

Next, the following theorem shows that the effective dimension of the linear
spline kernel is indeed
$\cO(T/\sqrt{\lambda})$.
\begin{theorem}\label{thm:spline-deff}
  Let $K_{T}\in\R^{T\times T}$ be the matrix with entries
  $[K_{T}]_{ij}=\min(i,j)$. Then
  \begin{align*}
    \deff(\lambda):=\Tr{K_{T}(\lambda I + K_{T})^{\inv}} \le \frac{\pi T}{2\sqrt{\lambda}}.
  \end{align*}
\end{theorem}
\begin{proof}
  By \Cref{lemma:spline-inverse}, the inverse of a matrix $K_{T}$
  with entries $[K_{T}]_{ij}=\min(i,j)$
  is the tri-diagonal matrix of the form
  \begin{align}
    K_{T}^{\inv}=\pmat{
    2&-1&0&0&\dots&0\\
    -1&2&-1&0&\dots&0\\
    0&-1&2&-1&\dots&0\\
    0&0&-1&2&\dots&0\\
    \vdots&&&&\ddots&\\
    0&0&0&0&\dots&1
    }.\label{mat:tri-diagonal}
  \end{align}
  The eigenvalues of matrices of this form are well-known
  \citep{rutherford1948continuant,losonczi1992eigenvalues,da2020eigenpairs} and have a closed form expression:
  \begin{align*}
    \lambda_{k}(K_{T}^{\inv}) = 2\brac{1-\cos\brac{\frac{2k\pi}{(2T+1)}}} = 4\sin^{2}\brac{\frac{k\pi}{2T+1}},
  \end{align*}
  where the second equality uses the identity $1-\cos(x)=2\sin^{2}(x/2)$.
  Moreover, using the fact that $\sin(x)$ is concave on $[0,\pi/2]$
  we can bound
  $\sin(x)\ge \frac{2}{\pi}x$, so the eigenvalues of $K_{T}^{\inv}$ can be bounded as
  \begin{align*}
    \lambda_{k}(K_{T}^{\inv})=4\sin^{2}\brac{\frac{k\pi}{2T+1}}\ge 4\cdot\frac{4}{\pi^{2}}\cdot\frac{\pi^{2}k^{2}}{(2T+1)^{2}}\ge \frac{k^{2}}{T^{2}}
  \end{align*}

  Thus,
  via direct calculation of the effective dimension
  $\deff(\lambda)=\Tr{K_{T}(\lambda I + K_{T})^{\inv}}=\sum_{k=1}^{T}\frac{\lambda_{k}(K_{T})}{\lambda_{k}(K_{T})+\lambda}$,
  we have
  \begin{align*}
    \deff(\lambda)
    &=
      \sum_{k=1}^{T}\frac{\lambda_{k}(K_{T})}{\lambda_{k}(K_{T})+\lambda}
      =
      \sum_{k=1}^{T}\frac{1}{1+\lambda/\lambda_{k}(K_{T})}=\sum_{k=1}^{T}\frac{1}{1+\lambda\lambda_{k}(K^{\inv}_{T})}\\
    &\le
      \sum_{k=1}^{T}\frac{1}{1+\lambda k^{2}/T^{2}}
    \le
      \int_{0}^{T}\frac{1}{1+\frac{\lambda}{T^{2}}x^{2}}dx
    \overset{(a)}{=}
      \frac{T}{\sqrt{\lambda}}\int_{0}^{\sqrt{\lambda}}\frac{1}{1+u^{2}}du\\
    &\overset{(b)}{=}
      \frac{T}{\sqrt{\lambda}}\arctan(x)\Big|_{x=0}^{\sqrt{\lambda}}
    \overset{(c)}{\le} \frac{\pi T}{2\sqrt{\lambda}}
  \end{align*}
  where $(a)$ makes a change of variables $x= T/\sqrt{\lambda} u$,
  $(b)$ uses the fact that
  $\int_{a}^{b}\frac{1}{1+u^{2}}du = \arctan(x)\Big|_{a}^{b}$ and $(c)$ uses
  $\abs{\arctan(x)}\le \pi/2$ for all $x$.
\end{proof}

\section{Proofs for Section~\ref{sec:directional} (\SecDirectional)}%
\label{app:directional}

In this section we provide the full statement and proof of
\Cref{prop:simple-full-matrix}. 
\begin{restatable}{proposition}{FullMatrixBound}\label{prop:full-matrix-bound}
  Let $\ell_{1},\ldots,\ell_{T}$ be a sequence of $G$-Lipschitz losses and for all $t$,
  let $\gt\in\partial\ell_{t}(\wt)$ and define $\Gt = \gt\otimes \phi(t)\in\WW$,
  $G_{0}=G\sqrt{\max_{t}k(t,t)}$ and
  $\Sigma_{t}=(\lambda+G_{0}^{2}) I + \sum_{s=1}^{\tmm}\Gt\otimes\Gt$.\footnote{Here,
  the tensor product $\Gt\otimes \Gt$ is the map such that for
  $V\in\Lin(\hh,\ww)$,
  $(\Gt\otimes\Gt)(V)=\inner{\Gt,V}_{\knorm}\Gt\in \Lin(\hh,\ww)$.
  Note that $\Sigma_{t}$ is a self-adjoint operator.
}
  Let $(\norm{\cdot}_{t},\norm{\cdot}_{t,*})$ be a dual-norm pair
  characterized by
  $\norm{\W}_{t}=\sqrt{\inner{\W,\Sigma_{t}\W}_{\knorm}}$.
  Let $\epsilon>0$, and for all $t$ let $V_{t}=4G_{0}^{2} + \sum_{s=1}^{\tmm}\norm{\Gt}^{2}_{t,*}$,
  $\alpha_{t}=\frac{\epsilon G_{0}}{\sqrt{V_{t}}\log^{2}(V_{t}/G_{0}^{2})}$
  and set
  $\psi_{t}(\W)=k\int_{0}^{\norm{\W}_{t}}\min_{\eta\le 1/G_{0}}\sbrac{\frac{\Log{x/\alpha_{t}+1}}{\eta}+\eta V_{t}}dx$.

  Suppose on each round we set
  $\Wt =\arg\min_{\W\in \WW}\inner{\sum_{s=1}^{\tmm}\Gvar{s},\W}+\psi_{t}(\W)$
  and we play $\wt = \Wt\phi(t)$. Then
  for $\cmp_{1},\ldots,\cmp_{T}$ in $\ww$ and $\Cmp\in \WW$ satisfying
  $\cmp_{t}=\Cmp\phi(t)$ for all $t$,
  \begin{align*}
    \RKHSRegret(\Cmp)
    &=
      \tilde{\mathcal{O}}\brac{\epsilon G_{0} + \sqrt{\deff(\lambda)\sbrac{(\lambda +G_{0}^{2})\norm{\Cmp}^{2}_{\knorm}+\sumtT \inner{\gt,\cmp_{t}}^{2}}\Log{e+\frac{e \lambda_{\max}(K_{T})}{\lambda}}}}
  \end{align*}
  where $K_{T}=(\inner{\gt,\gs}k(t,s))_{t,s\in [T]}$, and
  $\deff(\lambda)=\Tr{K_{T}(\lambda I + K_{T})^{\inv}}$
\end{restatable}
\begin{proof}
  The result follows as a special case of \Cref{thm:static-full-matrix} with
  sequence of non-decreasing norms characterized by
  $\norm{W}_{t}=\sqrt{\inner{W,\Sigma_{t}W}_{\cW}}$ and Lipschitz constant
  $G_{0}=G\sqrt{\max_{t}k(t,t)}\ge \norm{\grad\tilde{\ell}_{t}(\Wt)}$ for
  all $t$.
  First, observe that
  \begin{align*}
    \inner{\Gt,\Cmp}^{2}_{\knorm}=\inner{\gt\otimes\phi(t),\Cmp}^{2}_{\knorm}=\inner{\gt,\Cmp\phi(t)}_{\ww}^{2}=\inner{\gt,\cmp_{t}}_{\ww}^{2},
  \end{align*}
  hence from the static regret guarantee of \Cref{thm:static-full-matrix}, we get
  \begin{align*}
    \RKHSRegret(\Cmp)
    &= \tilde{\mathcal{O}}\brac{\epsilon G_{0} + \norm{\Cmp}_{T+1}\sqrt{\sumtT \norm{\Gt}^{2}_{t,*}}}\\
    &=
      \tilde{\mathcal{O}}\brac{\epsilon G_{0} + \sqrt{\brac{(\lambda + G_{0}^{2})\norm{\Cmp}_{\knorm}^{2}+\sumtT \inner{\Gt,\Cmp}_{\knorm}^{2}}\sumtT \norm{\Gt}^{2}_{t,*}}}\\
    &=
      \tilde{\mathcal{O}}\brac{\epsilon G_{0} + \sqrt{\brac{(\lambda+G_{0}^{2})\norm{\Cmp}_{\knorm}^{2}+\sumtT \inner{\gt,\cmp_{t}}^{2}}\sumtT \norm{\Gt}^{2}_{t,*}}}.
  \end{align*}
  Moreover, observing that
  \begin{align*}
    \norm{\Gt}^{2}_{t,*}
    &=
      \inner{\Gt, \brac{(\lambda+G_{0}^{2})I+\sum_{s=1}^{\tmm}\Gs\otimes\Gs}^{\inv}\Gt}
    \le
      \inner{\Gt,\brac{\lambda I + \sum_{s=1}^{t}\Gs\otimes\Gs}^{\inv}\Gt},
  \end{align*}
  we have via \Cref{lemma:log-det-sum} that
  \begin{align*}
    \sumtT \norm{\Gt}^{2}_{t,*}
    &\le
      \deff(\lambda)\Log{e+\frac{e\lambda_{\max}(K_{T})}{\lambda}},
  \end{align*}
  where $K_{T}$ is the gram matrix with entries
  $[K_{T}]_{ij}= \inner{g_{i},g_{j}}k(i,j)$ and
  $\deff(\lambda)=\Tr{K_{T}(\lambda I +K_{T})^{-1}}=\sum_{k=1}^{T}\frac{\lambda_{k}(K_{T})}{\lambda+\lambda_{k}(K_{T})}$.

  Hence the dynamic regret 
  $\DRegret(\cmp_1,\ldots,\cmp_T)=\RKHSRegret(\Cmp)$ can be bound above by
  \begin{align*}
      \tilde{\mathcal{O}}\brac{\epsilon G_{0} + \sqrt{\deff(\lambda)\sbrac{(\lambda + G_{0}^{2})\norm{\Cmp}^{2}_{\knorm}+\sumtT \inner{\gt,\cmp_{t}}^{2}}\Log{1+\frac{\lambda_{\max}(K_{T})}{\lambda}}}}.
  \end{align*}

\end{proof}

\subsection{Directional Adaptivity via Varying-norms}

For completeness we provide a mild generalization of
the static regret algorithm of \cite{jacobsen2022parameter}
to leverage an arbitrary sequence of
increasing norms. A similar technique has been used
to get full-matrix parameter-free rates by
\cite{cutkosky2020better}.

The analysis remains mostly the same as \textcite{jacobsen2022parameter}, but
their analysis of the stability term
bounds
$-D_{\psi_{t}}(\wtpp|\wt)$
via a lemma that assumes that $\psi_{t}(\wvec)=\Psi_{t}(\norm{\wvec}_{2})$
for $w\in\R^{d}$.
To obtain a full-matrix version of their result,
we would instead like to have $\Psi_{t}(\norm{w}_{M})$,
where $\norm{\cdot}_{M}$ is
a weighted norm \wrt{} to the inner product
$\inner{\cdot,\cdot}_{\cW}$ on an arbitrary Hilbert space $\cW$.
In what follows, we drop the dependence on $\cW$
for brevity and simply write $\inner{\cdot,\cdot}$.

We first state and prove the main result of this section. The
proof will rely on a few technical lemmas, which we state and prove at the end
of the section in \Cref{app:stability}.

\begin{restatable}{theorem}{StaticFullMatrix}\label{thm:static-full-matrix}
  Let $\cW$ be a Hilbert space and let $\inner{\cdot,\cdot}$
  denote the associated inner product,
  let $\norm{\cdot}_{1},\ldots,\norm{\cdot}_{T+1}$ be an arbitrary
  sequence of non-decreasing norms on $\cW$, and let
  $\norm{\cdot}_{0}:=\sqrt{\inner{\cdot,\cdot}}\le \norm{\cdot}_{t}$ for all $t$.
  Let $\ell_{1},\ldots,\ell_{T}$ be convex functions over $\cW$
  satisfying $\norm{\gt}_{t,*}\le G$ for all $t$ and $g_{t}\in\partial\ell_{t}(\wt)$.
  Let $\epsilon,\lambda >0$,
  $V_{t}= 4 G^{2} +\sum_{s=1}^{\tmm}\norm{g_{s}}^{2}_{s,*}$,
  $\alpha_{t}=\frac{\epsilon G}{\sqrt{V_{t}}\log^{2}(V_{t}/G^{2})}$,
  and set
  \(
    \psi_{t}(\w)
    =
      3\int_{0}^{\norm{\w}_{t}}\min_{\eta\le\frac{1}{G}}\sbrac{\frac{\Log{x/\alpha_{t}+1}}{\eta}+\eta V_{t}}dx
      \), and
      on each round update $\wtpp = \arg\min_{w\in \cW}\inner{\sum_{s=1}^{t}g_{s},w}+\psi_{\tpp}(w)$.
  Then
  for all $\cmp\in\cW$,
  \begin{align*}
    R_{T}(\cmp)
    &=
      \Ohat\brac{G\epsilon + \norm{\cmp}\sbrac{\sqrt{V_{T}\Log{\frac{\norm{\cmp}\sqrt{V_{T}}}{\epsilon G}+1}}\maxOp G\Log{\frac{\norm{\cmp}\sqrt{V_{T}}}{\epsilon G}+1}}}
  \end{align*}
  where $\Ohat(\cdot)$ hides constant and $\log(\log)$ factors (but not $\log$ factors).
\end{restatable}
\begin{proof}
  Begin by applying the standard FTRL regret template (see, \eg{},
  \textcite[Lemma 7.1]{orabona2019modern}):
  \begin{align*}
    R_{T}(\cmp)
    &=
      \sumtT \inner{\gt, \wt-\cmp}
    \le
      \psi_{T+1}(\cmp)+\sumtT F_{t}(\wt)- F_{\tpp}(\wtpp)+\inner{\gt,\wt},
  \end{align*}
  where $F_{t}(w)=\inner{\sum_{s=1}^{\tmm}g_{s},w}+\psi_{t}(\w)$.
  Observe that the summation can be written as
  \begin{align*}
    &\sumtT F_{t}(\wt)-F_{\tpp}(\wtpp)+\inner{\gt,\wt}\\
    &\qquad=
      \sumtT \inner{\gt,\wt-\wtpp}+F_{t}(\wt)-F_{t}(\wtpp)+(\psi_{t}-\psi_{\tpp})(\wtpp)\\
    &\qquad\overset{(a)}{=}
      \sumtT \inner{\gt,\wt-\wtpp}+\inner{\grad F_{t}(\wt),\wt-\wtpp}-D_{F_{t}}(\wtpp|\wt)
      +(\psi_{t}-\psi_{\tpp})(\wtpp)\\
    &\qquad\overset{(b)}{\le}
      \sumtT \inner{\gt,\wt-\wtpp}-D_{F_{t}}(\wtpp|\wt)+(\psi_{t}-\psi_{\tpp})(\wtpp)\\
    &\qquad\overset{(c)}{=}
      \sumtT \norm{\gt}_{t,*}\norm{\wt-\wtpp}_t -D_{\psi_{t}}(\wtpp|\wt)-(\psi_{\tpp}-\psi_{t})(\wtpp),
  \end{align*}
  where $(a)$ uses the definition of Bregman divergence to write
  $f(a)-f(b)=\inner{\grad f(a), a-b}-D_{f}(b|a)$, $(b)$ uses the fact that $\wt=\arg\min_{w\in \cW}F_{t}(\w)$, hence
  $\inner{\grad F_{t}(\wt),\wt-\wtpp}\le 0$ by the first-order optimality
  condition, and $(c)$ uses the fact that Bregman divergences are invariant to
  linear terms, so from the definition of $F_{t}$ we have $D_{F_{t}}(\cdot|\cdot)=D_{\psi_{t}}(\cdot|\cdot)$. Moreover, since
  $\norm{\cdot}_{1},\ldots,\norm{\cdot}_{T}$ is a non-decreasing sequence of norms, we can bound the terms
  \begin{align*}
    (\psi_{\tpp}-\psi_{t})(\w)
    &=
      \Psi_{\tpp}(\norm{\w}_{\tpp})-\Psi_{t}(\norm{\w}_{t})
    \ge
      \underbrace{\Psi_{\tpp}(\norm{\w}_{t})-\Psi_{t}(\norm{\w}_{t})}_{
      =: \Delta_{t}^{\Psi}(\norm{w}_{t})},
  \end{align*}
  so overall
  the regret is bounded by
  \begin{align*}
    R_{T}(\cmp)
    &\le
      \psi_{T+1}(\cmp)+\sumtT\underbrace{ \norm{\gt}_{t,*}\norm{\wt-\wtpp}_{t}-D_{\psi_{t}}(\wtpp|\wt)-\Delta_{t}^{\Psi}(\norm{\wtpp}_{t})}_{=:\delta_{t}}.
  \end{align*}
  From here, the rest of the proof follows using the same arguments as
  \cite{jacobsen2022parameter}, but using our
  \Cref{lemma:elliptically-symmetric-bound} to bound
  $D_{\psi_{t}}(\wtpp|\wt)\ge \half \norm{\wt-\wtpp}^{2}\Psi_{t}(\norm{\tilde w}_{t})$
  instead of their Lemma 7.
\end{proof}

\subsubsection{A Stability Lemma for Weighted Norms}%
\label{app:stability}

In this section generalize the stability lemma of \citet{jacobsen2022parameter}
to weighted norms $\norm{x}_{M}=\sqrt{\inner{x, Mx}}$. This is
the main technical detail needed for the proof of \Cref{thm:static-full-matrix} that is not covered by the proof
of their static regret algorithm.
Throughout this section we
assume the domain $\cW$ is a Hilbert space with associated inner product
$\inner{\cdot,\cdot}$. The following helper lemma follows via
a straight-forward but somewhat tedious computation.
\begin{restatable}{lemma}{TediousHessian}\label{lemma:tedious-hessian}
  Let $g:\cW\to\R$ be a convex function and
  let $f(x)=\sqrt{g(x)}$. Then for $x\in\cW$  s.t.
  $g(x)>0$ we have
  \begin{align*}
    \grad f(x) = \frac{\grad g(x)}{2f(x)}\qquad\text{and}\qquad
    \grad^{2}f(x) = \frac{\grad^{2}g(x)}{2f(x)}-\frac{\grad g(x)\otimes \grad g(x)}{4 f(x)^{3}},
  \end{align*}
  where $\otimes$ denotes the tensor product.
\end{restatable}
Using this, we have the following Hessian bounds for elliptically-symmetric
functions:
\begin{restatable}{lemma}{EllipticallySymmetricBound}\label{lemma:elliptically-symmetric-bound}
 Let $M\in\Lin(\hh,\hh)$ be a positive definite linear operator and assume $M$
 is self-adjoint \wrt{} $\inner{\cdot,\cdot}$. Let
 $\norm{x}_{M}=\sqrt{\inner{x,Mx}}$ be the weighted norm induced by $M$ and
 let $\psi(w)=\Psi(\norm{\w}_{M})$ for some convex function $\Psi:\R\to\R$.
 Then for any $\w\in\cW$ bounded away from $0$ and any $u\in\cW$,
 \begin{align*}
   \inner{u,\grad^{2}\psi(\w)u}\ge \Min{\Psi''(\norm{\w}_{M}),\frac{\Psi'(\norm{\w}_{M})}{\norm{\w}_{M}}}\norm{u}_{M}^{2}
 \end{align*}
 Moreover, if $\Psi'(\cdot)$ is concave and non-negative, then for any $\w\in \cW$ bounded
 away from $0$ and $u\in\cW$,
 \begin{align*}
   \inner{u,\grad^{2}\psi(w)u}\ge \Psi''(\norm{\w}_{M})\norm{u}^{2}_{M}~.
 \end{align*}
\end{restatable}
\begin{proof}
  The proof follows a similar argument to
  \citet[Lemma 23]{orabona2021parameterfree}.
  Let us first compute the gradients of $f(x)=\norm{x}_{M}=\sqrt{\inner{x,Mx}}$.
  Let $g(x)=\inner{x,Mx}$ and observe that if $M$ is self-adjoint \wrt{}
  $\inner{\cdot,\cdot}$, we have $\grad g(x)=2Mx$ and
  $\grad^{2}g(x)=2M$.
  Hence, applying \Cref{lemma:tedious-hessian} we have
  \begin{align*}
    \grad f(\w)&= \frac{2M\w}{2\norm{\w}_{M}}=\frac{M\w}{\norm{\w}_{M}}\\
    \grad^{2}f(\w)&=\frac{2M}{2\norm{\w}_{M}}-\frac{4 M\w\otimes M\w}{4 \norm{\w}^{3}_{M}}
                   =\frac{M}{\norm{\w}_{M}}-\frac{M\w\otimes M\w}{\norm{\w}^{3}_{M}}.
  \end{align*}
  Using this, we have
  \begin{align*}
    \grad\psi(\w)=\grad\Psi(\norm{\w}_{M})=\grad\norm{\w}_{M}\Psi'(\norm{\w}_{M})=\frac{Mw}{\norm{\w}_{M}}\Psi'(\norm{\w}_{M}),
  \end{align*}
  and
  \begin{align*}
    \grad^{2}\psi(\w)
    &=
    \grad\brac{\grad\norm{\w}_{M}\Psi_{t}'(\norm{\w}_{M})}\\
    &=
      \grad^{2}\norm{\w}_{M}\Psi'(\norm{\w}_{M})+\Psi''(\norm{\w}_{M})\brac{\grad\norm{\w}_{M}\otimes\grad\norm{\w}_{M}}\\
    &=
      \Psi_{t}'(\norm{\w}_{M})\brac{\frac{M}{\norm{\w}_{M}}-\frac{(M\w\otimes M\w)}{\norm{\w}^{3}_{M}}}+\Psi''(\norm{\w}_{M})\frac{(M\w\otimes M\w)}{\norm{\w}_{M}^{2}}\\
    &=
      \underbrace{\brac{\frac{\Psi''(\norm{\w}_{M})}{\norm{\w}_{M}^{2}}-\frac{\Psi'(\norm{\w}_{M})}{\norm{\w}_{M}^{3}}}}_{=:\beta}(M\w\otimes M\w)+\underbrace{\frac{\Psi'(\norm{\w}_{M})}{\norm{\w}_{M}}}_{=:\gamma}M\\
    &=
      \beta (Mw\otimes Mw)+\gamma M.
  \end{align*}
  Hence, for any $u\in\cW$
  we have
  \begin{align*}
    \inner{u, \grad^{2}\psi(w)u}
    &=
      \inner{u, (\beta M\w\otimes M\w+\gamma M)u}\\
      &=
        \beta\inner{u,M\w}^{2}+\gamma \norm{u}^{2}_{M}
  \end{align*}
  Now decompose $u=w+v$ for some $v$ such that
  $\inner{Mw,v}=0$; such a $v$ always exists for positive definite $M$. Then
  \begin{align*}
    \inner{u,\grad^{2}\psi(w)u}
    &=
      \beta \inner{v+w,Mw}^{2}+\gamma\norm{v+w}^{2}_{M}
      =
      \beta\norm{w}^{4}_{M}+\gamma\norm{\w}^{2}_{M}+\gamma \norm{v}^{2}_{M}\\
    &=
      \brac{\frac{\Psi''(\norm{\w}_{M})}{\norm{\w}^{2}_{M}}-\frac{\Psi'(\norm{\w}_{M})}{\norm{\w}^{3}_{M}}}\norm{\w}^{4}_{M}+\frac{\Psi'(\norm{\w}_{M})}{\norm{\w}_{M}}\brac{\norm{\w}_{M}^{2}+\norm{v}^{2}_{M}}\\
    &=
      \Psi''(\norm{\w}_{M})\norm{\w}^{2}_{M}+\frac{\Psi'(\norm{\w}_{M})}{\norm{\w}_{M}}\norm{v}^{2}_{M}\\
    &\ge
      \Min{\Psi''(\norm{\w}_{M}), \frac{\Psi'(\norm{\w}_{M})}{\norm{\w}_{M}}}\brac{\norm{\w}^{2}_{M}+\norm{v}^{2}_{M}}\\
    &=
      \Min{\Psi''(\norm{\w}_{M}), \frac{\Psi'(\norm{\w}_{M})}{\norm{\w}_{M}}}\norm{u}^{2}_{M},
  \end{align*}
  where the last step uses the fact that $w$ and $v$ are orthogonal \wrt{} $M$.

  For the second statement of the lemma, we need
  only show that $\Psi'(\norm{\w}_{M})/\norm{\w}_{M}\ge\Psi''(\norm{\w}_{M})$.
  This is indeed the case by concavity and non-negativity of $\Psi'(\cdot)$:
  \begin{align*}
    \frac{\Psi'(\norm{\w}_{M})}{\norm{\w}_{M}}\ge \frac{\Psi'(0)+\Psi''(\norm{\w}_{M})(\norm{\w}_{M}-0)}{\norm{\w}_{M}} \ge \Psi''(\norm{\w}_{M})~. \tag*{\qedhere}
  \end{align*}
\end{proof}

\section{Supporting Lemmas}

The following lemma is a straight-forward generalization of the usual
log-determinant lemma (see, \textit{e.g.}, \cite[Lemma
11.11]{cesa2006prediction}), taking a bit of extra care to
handle determinants of potentially infinite-dimensional linear operators.
\begin{restatable}{lemma}{LogDet}\label{lemma:log-det}
  Let $\hh$ be a Hilbert space and for all $t$ let $v_{t}\in \hh$. Suppose
  $v_{t}\otimes v_{t}:\hh\to \hh$ defines a bounded linear operator for all $t$
  and suppose $A_{t}= A_{\tmm}+v_{t}\otimes v_{t}$ for any $t\ge 1$, starting from $A_{0}=I$.
    Then
  \begin{align*}
    \inner{v_{t}, A_{t}^{\inv}v_{t}} = 1-\frac{\Det{A_{\tmm}}}{\Det{A_{t}}}.
  \end{align*}
\end{restatable}
\begin{proof}
  Observe that for any $t$, we have $A_{t}=A_{\tmm}+v_{t}\otimes v_{t}$,
  so re-arranging terms, factoring, and taking determinants of both sides we have
  \begin{align}
    \Det{A_{t}}\Det{I-A_{t}^{\inv}v_{t}\otimes v_{t}} = \Det{A_{\tmm}}.\label{eq:log-det}
  \end{align}
  Note that each of these determinants are well-defined in terms of the Fredholm
  determinant: each of the three terms above is a trace-class
  perturbation of the identity operator.
  Moreover, observe that $A_{t}^{\inv}(v_{t}\otimes v_{t})$ is a rank-one
  operator having single eigenvalue equal to
  $\lambda=\inner{v_{t},A_{t}^{\inv}v_{t}}$. Indeed,
  for any $w\in \hh$ we have
    $A_{t}^{\inv}(v_{t}\otimes v_{t})(w)=\inner{v_{t},w}A_{t}^{\inv}v_{t}$,
  hence, for $w=A_{t}^{\inv}v_{t}$ we have
 \begin{align*}
   A_{t}^{\inv}(v_{t}\otimes v_{t})(w)=A_{t}^{\inv}(v_{t}\otimes v_{t})(A_{t}^{\inv}v_{t})=\inner{v_{t},A_{t}^{\inv}v_{t}}A_{t}^{\inv}v_{t}=\lambda w~.
 \end{align*}
 Therefore, from the standard rank-one perturbation identity for the determinant
 we have $\Det{I-A_{t}^{\inv}v_{t}\otimes v_{t}}=1-\inner{v_{t},A_{t}^{\inv}v_{t}}$,
 so re-arranging \Cref{eq:log-det} yields
 \[
   \inner{v_{t},A_{t}^{\inv}v_{t}}=1-\frac{\Det{A_{\tmm}}}{\Det{A_{t}}}~. \qedhere
 \]
\end{proof}

\begin{restatable}{lemma}{LogDetSum}\label{lemma:log-det-sum}
  Let $\hh$ be a Hilbert space. For all $t$ let $\Gt\in \hh$ be a bounded linear
  operator and define
  $S_{t}=\lambda I + \sum_{s=1}^{t}\Gs\otimes\Gs$ for $\lambda >0$. Then
  \begin{align*}
    \sumtT \inner{\Gt,S_{t}^{\inv}\Gt}\le
      \deff(\lambda)\Log{e+\frac{e\lambda_{\max}(K_{T})}{\lambda}},
  \end{align*}
  where $K_{T}=(\inner{\gt,\gvec_{s}}k(t,s))_{t,s\in[T]}$ and
  $\deff(\lambda)=\Tr{K_{T}(\lambda I + K_{T})^{\inv}}$.
\end{restatable}
\begin{proof}
  First apply \Cref{lemma:log-det} with $v_{t}=\Gt/\sqrt{\lambda}$ followed by the elementary inequality
  $1-x\le \Log{1/x}$ to get
  \begin{align*}
    \sumtT \inner{\Gt,S_{t}^{\inv}\Gt}
    &=
      \sumtT \frac{1}{\lambda}\inner{\Gt,\brac{S_{t}/\lambda}^{\inv}\Gt}\\
    &=
      \sumtT 1-\frac{\Det{S_{\tmm}/\lambda}}{\Det{S_{t}/\lambda}}\\
    &\le
      \sumtT \Log{\frac{\Det{S_{t}/\lambda}}{{\Det{S_{\tmm}/\lambda}}}}\\
    &=
      \Log{\Det{I+\frac{\sumtT \Gt\otimes\Gt}{\lambda}}}\\
    &=
      \sum_{t=1}^{T}\Log{1+\frac{\lambda_{t}(K_{T})}{\lambda}},
  \end{align*}
  where the last line uses the well-known fact that the gram matrix
  $K_{T}=(\inner{g_{t},g_{s}}k(t,s))_{t,s\in[T]}$ and the
  empirical covariance operator $\sumtT \Gt\otimes\Gt$ have the same
  eigenvalues. Moreover, following \textcite[Proposition 2]{jezequel2019efficient} we can use the inequality
  $\Log{1+x}\le \frac{x}{1+x}(1+\Log{1+x})$
  to expose a dependence on the effective dimension
  $\deff(\lambda)=\Tr{K_{T}(\lambda I + K_{T})^{\inv}}$ as follows:
  \begin{align*}
    \sumtT \inner{\Gt,S_{t}^{\inv}\Gt}
    &\le
      \sumtT\frac{\lambda_{t}(K_{T})}{\lambda + \lambda_{t}(K_{T})}\sbrac{1+\Log{1+\frac{\lambda_{t}(K_{T})}{\lambda}}}\\
    &\le
      \sbrac{1+\Log{1+\frac{\lambda_{\max}(K_{T})}{\lambda}}}\sumtT \frac{\lambda_{t}(K_{T})}{\lambda + \lambda_{t}(K_{t})}\\
    &=
      \sbrac{1+\Log{1+\frac{\lambda_{\max}(K_{T})}{\lambda}}}\Tr{K_{T}(\lambda I + K_{T})^{\inv}}\\
    &=
      \deff(\lambda)\sbrac{1+\Log{1+\frac{\lambda_{\max}(K_{T})}{\lambda}}}\\
    &=
      \deff(\lambda)\Log{e+\frac{e\lambda_{\max}(K_{T})}{\lambda}}~. \qedhere
  \end{align*}

\end{proof}

\begin{restatable}{lemma}{FeatureCovarianceNorm}\label{lemma:feature-covariance-norm}
  Let $\hh$ be a separable RKHS with associated feature map $\phi(t)\in \hh$ and let
  $x\in\hh$ satisfy $x(t)=X\phi(t)$ for some $X\in \Lin(\hh,\ww)$.
  Then
  \begin{align*}
    \norm{x(t)}^{2}_{\ww} = \inner{X(\phi(t)\otimes\phi(t)), X}_{\knorm},
  \end{align*}
  where $\phi(t)\otimes\phi(t):\hh\to\hh$ is the linear operator with action
  $(\phi(t)\otimes\phi(t))h = \inner{\phi(t),h}_{\hh}\phi(t)$.
\end{restatable}
\begin{proof}
  Let $h_{1},h_{2},\ldots$ be an orthonormal basis of $\hh$. By definition of the
  Hilbert-Schmidt inner product,
  we have
  \begin{align*}
    \inner{X(\phi(t)\otimes\phi(t)),X}_{\knorm}
    &=
      \sum_{i}\inner{X(\phi(t)\otimes\phi(t))h_{i},X h_{i}}_{\ww}\\
    &=
      \sum_{i}\inner{\phi(t),h_{i}}_{\hh}\inner{X\phi(t),X h_{i}}_{\ww}\\
    &=
      \sum_{i}\inner{\phi(t),h_{i}}_{\hh}\inner{X^{*}X\phi(t),h_{i}}_{\hh}\\
    &\overset{(\star)}{=}
      \inner{\phi(t),X^{*}X\phi(t)}_{\hh} = \inner{X\phi(t),X\phi(t)}_{\ww}\\
    &=\inner{x,x}_{\ww}=\norm{x}^{2}_{\ww},
  \end{align*}
  where $X^{*}:\ww\to\hh$ is the adjoint of $X$ and $(\star)$ uses Parseval's identity.
\end{proof}

\begin{restatable}{lemma}{HSSumNorm}\label{lemma:hs-sum-norm}
Let $\ww\subseteq\R^d$ and let $\hh$ be an RKHS with associated feature map $\phi$. For all $t\in[T]$, let $\Gt=\gt\otimes\phi(t)\in\Lin(\hh,\ww)$ denote
the rank-one operator mapping $\Gt(h) = \inner{\phi(t),h}\gt\in\ww$. Then for any $t$,
\begin{align*}
    \norm{\sum_{s=1}^t\Gs}_\knorm^2 = \sum_{s,s'}^{t}k(s,s')\inner{\gvar{s},\gvar{s'}}_\ww~.
\end{align*}
\end{restatable}
\begin{proof}
Let $h_1,h_2,\ldots$ be an orthonormal basis of $\hh$. Observe that 
for any $h\in\hh$, $\brac{\sum_{s=1}^t \Gs}(h) = \sum_{s=1}^t \inner{\phi(s), h}g_s$. Hence, by definition of the Hilbert-Schmidt norm,
\begin{align*}
\norm{\sum_{s=1}^t\Gvar{t}}_\knorm^2
&=
    \sum_{i}\norm{\sum_{s=1}^t\Gvar{t}h_i}^2_\ww
    =
    \sum_{i}\inner{\sum_{s=1}^t\inner{\phi(s),h_i}_\hh g_s, \sum_{s'=1}^t\inner{\phi(s'),h_i}_\hh g_{s'}}_\ww\\
    &=
    \sum_{i}\sum_{s,s'}^t\inner{\phi(s),h_i}_\hh\inner{\phi(s'),h_i}_\hh\inner{g_s,g_s'}_\ww\\
    &=
    \sum_{s,s'}\inner{\gvar{s},\gvar{s'}}_\ww \sum_i \inner{\phi(s),h_i}_\hh\inner{\phi(s'),h_i}_\hh\\
    &=
    \sum_{s,s'}\inner{\gvar{s},\gvar{s'}}_\ww k(s,s'),
\end{align*}
where the last line observes that for orthonormal basis $h_i$ we have $\sum_i \inner{\phi(s),h_i}_\hh\inner{\phi(s'), h_i}_\hh = \inner{\phi(s),\phi(s')}_\hh=k(s,s')$.
\end{proof}

The following theorem shows how to compute the norm of
$\Gt=\gt\otimes\phi(t)$, which is the auxiliary loss for OLO under our framework.
Here we state the result in terms of $\gt\in \ww^{*}$ for generality,
but note that in the main text we implicitly invoke Riesz
representation theorem to write $\gt\in\ww$, $\Gt\in\Lin(\hh,\ww)$, and $\norm{\Gt}=\norm{\gt}_{\ww}\sqrt{k(t,t)}$.
\begin{restatable}{lemma}{HSNorm}\label{lemma:hs-norm}
  Let $\hh$ be a RKHS with associated feature map $\phi(t)$
  and let $\ww$ be a Hilbert space.
  Let $\ell_{t}:\ww\to\R$ be a differentiable function
  and for any $\W\in\WW$ let $\tilde\ell_{t}(\W)=\ell_{t}(\W\phi(t))$.
  Then for any $\W\in\Lin(\hh,\ww)$, $\gt\in\partial\ell_{t}(\W\phi(t))$, and $\Gt=\gt\otimes\phi(t)\in\partial\elltilde_{t}(\W)$,
  \begin{align*}
    \norm{\Gt}_{\knorm}^{2}= \norm{\gt}_{\ww,*}^{2}k(t,t),
  \end{align*}
  where $k(s,t)=\inner{\phi(s),\phi(t)}_{\hh}$ is the kernel associated with
  $\hh$ and
  $\norm{\cdot}_{\ww,*}$ is the dual norm of $\norm{\cdot}_{\ww}$.
\end{restatable}
\begin{proof}
  We have via \Cref{lemma:subgradients} that
  \begin{align*}
    \Gt:= g_{t}\otimes \phi(t)\in \partial\elltilde_{t}(\W)\subseteq\Lin(\hh,\ww)^{*},
  \end{align*}
  where $ g_{t}\in\partial\ell_{t}(\W\phi(t))\subseteq\ww^{*}$.
  By Riesz representation theorem,
  we can identify a $\hat\gt\in\ww$ such that
  for any $\w\in\ww$, $\gt(\w)=\inner{\hat\gt,\w}_{\ww}$,
  and likewise
  we can identify
  $\Gt\in \Lin(\hh,\ww)^{*}$  with a rank-one operator
  $\hat\Gt\in\Lin(\hh,\ww)$ with action
  $\hat\Gt(h)=\inner{\phi(t),h}_{\hh}\hat\gt$.
  Hence, we have by definition of the Hilbert-Schmidt norm
  that for any orthonormal basis $\Set{h_{i}}_{i}$ of $\hh$,
  \begin{align*}
    \norm{\Gt}^{2}_{\knorm}
    &=
      \sum_{i}\norm{\Gt h_{i}}_{\ww}^{2}\\
    &=
      \sum_{i}\inner{\phi(t),h_{i}}_{\hh}^{2}\norm{\hat\gt}^{2}_{\ww}\\
    &=
      \norm{\gt}_{\ww,*}\norm{\phi(t)}^{2}_{\hh},
  \end{align*}
  where the last line again uses Riesz representation theorem to write $\norm{\hat\gt}_{\ww}=\norm{\gt}_{\ww,*}$
  and then uses $\sum_{i}\inner{\phi(t),h_{i}}_{\hh}^{2}=\norm{\phi(t)}^{2}_{\hh}$
  by Parseval's identity. Moreover,
  since $\phi(t)$ are the features of an RKHS with
  kernel $k$, we have
  \[
    \norm{\phi(t)}_{\hh}^{2}=\inner{\phi(t),\phi(t)}_{\hh}=k(t,t)~. \qedhere
  \]
\end{proof}

\begin{restatable}{lemma}{Subgradients}\label{lemma:subgradients}
  Let $\hh$ be a RKHS with feature map $\phi(t)\in \hh$,
  and let $\ww$ be a Hilbert space.
  Let $\ell_{t}:\ww\to \R$  be a convex function and
  let $\elltilde_{t}(\W)=\ell_{t}(\W\phi(t))$ for $\W\in\WW$.
  Then for any $\W\in\WW$ and any $\gt\in \partial \ell_{t}(\W\phi(t))\subseteq\ww^{*}$,
  \begin{align*}
    \Gt = \gt\otimes \phi(t)\in \partial\elltilde_{t}(\W)\in \Lin(\hh,\ww)^{*},
  \end{align*}
  where $\Gt\in \Lin(\hh,\ww)^{*}$ is the functional with
  action $\Gt(W)=\inner{\gt, W\phi(t)}_{\ww}$ for all $\W\in \WW$.
\end{restatable}
\begin{proof}
  Let $\W\in\WW$, $\wt=\W\phi(t)\in\ww$, and let $\gt\in \partial\ell_{t}(\wt)\subseteq\ww^{*}$.
  Define $G_{t}=\gt\otimes \phi(t)\in \Lin(\hh,\ww)^{*}$  the
  functional on $\Lin(\hh,\ww)$ with action
  \begin{align*}
    \Gt(\W)=\inner{\gt, \W\phi(t)}_{\ww},\quad\forall \W\in\WW.
  \end{align*}
  Now observe that for $\gt\in\partial\ell_{t}(\wt)$, for any $\w\in \ww$
  we have
  \begin{align*}
    \ell_{t}(\w)\ge \ell_{t}(\wt)+\inner{\gt,\w-\wt}_{\ww}
    = \ell_{t}(\W\phi(t))+\inner{\gt, \w-\W\phi(t)}_{\ww},
  \end{align*}
  hence for any $V\in \WW$ we can
  take $\w= V\phi(t)$ to get
  \begin{align*}
    \ell_{t}(V\phi(t)) \ge \ell_{t}(\W\phi(t)) + \inner{\gt, V\phi(t)-\W\phi(t)}_{\ww}
    = \ell_{t}(\W\phi(t))+\inner{\gt, (V-\W)\phi(t)}_{\ww}
  \end{align*}
  that is,
  \begin{align*}
    \elltilde_{t}(V)\ge \elltilde_{t}(W)+\Gt(V-\W).
  \end{align*}
  so $\Gt=\gt\otimes\phi(t)\in\partial\elltilde_{t}(\W)\subseteq\Lin(\hh,\ww)^{*}$.
\end{proof}

For completeness, the following lemma provides the inverse of
a matrix with entries $K_{ij}=\min(i,j)$. A similar result
can be seen in the proof of \citet[Lemma 4]{jacobsen2024equivalence}, where a
variant of the matrix $K$
appears as an intermediate calculation.
\begin{restatable}{lemma}{SplineInverse}\label{lemma:spline-inverse}
  Let $K\in\R^{T\times T}$ be a matrix with entries $K_{i,j}=\min(i,j)$. Then
  $K^{\inv}$ is a tri-diagonal matrix of the form
  \begin{align*}
    K^{\inv}=\pmat{
    2&-1&0&0&\dots&0&0\\
    -1&2&-1&0&\dots&0&0\\
    0&-1&2&-1&\dots&0&0\\
    0&0&-1&2&\dots&0&0\\
    \vdots&&&&\ddots&\\
    0&0&0&0&\dots&2&-1\\
    0&0&0&0&\dots&-1&1
    }.\label{mat:tri-diagonal}
  \end{align*}
\end{restatable}
\begin{proof}
  It can easily be checked that
  $K$ has Cholesky decomposition
  $K=U^{\top}U$ where $U$ is the upper-triangular
  matrix of $1's$. Hence,
  $K^{\inv}=U^{\inv} (U^{\top})^{\inv}$. Moreover, the inverse of $U$ is
  the first-order finite-differences operator
  with entries
  \begin{align*}
    \Sigma_{ij}=\begin{cases}1&\text{if }i=j\\ -1&\text{if }j=i+1\\0&\text{otherwise}\end{cases}.
  \end{align*}
  Indeed,
  $(U\Sigma)_{ij}=\sum_{k=1}^{T}U_{ik}\Sigma_{kj}= -U_{i,j-1} + U_{ij}=1$ for $i=j$ and
  zero otherwise.
  Computing
  $K^{\inv}=\Sigma\Sigma^{\top}$ yields the tri-diagonal matrix of the stated form.
\end{proof}

\end{document}